\documentclass[twoside,letterpaper]{article}
\usepackage[preprint]{aistats2027}
\let\acknowledgments\ackaccepted
\PassOptionsToPackage{dvipsnames}{xcolor}

\usepackage[american]{babel}

\usepackage[numbers,sort&compress]{natbib}
\usepackage{mathtools}
\usepackage{booktabs}
\usepackage{tikz}

\usepackage{microtype}
\usepackage{graphicx}
\usepackage{placeins}
\usepackage{subcaption}
\usepackage{booktabs}

\RequirePackage{algorithm}
\RequirePackage{algorithmic}

\usepackage{amsmath}
\usepackage{amssymb}
\usepackage{mathtools}
\usepackage{amsthm}

\usepackage{thmtools}

\theoremstyle{plain}
\newtheorem{theorem}{Theorem}[section]
\newtheorem{proposition}[theorem]{Proposition}
\newtheorem{lemma}[theorem]{Lemma}
\newtheorem{corollary}[theorem]{Corollary}
\theoremstyle{definition}
\newtheorem{definition}[theorem]{Definition}
\newtheorem{assumption}[theorem]{Assumption}
\theoremstyle{remark}
\newtheorem{remark}[theorem]{Remark}

\usepackage[utf8]{inputenc}
\usepackage[T1]{fontenc}

\usepackage{nccmath}
\usepackage{bm}
\usepackage{xurl} %
\usepackage{multirow}
\usepackage{amsfonts} 
\usepackage{comment}
\usepackage{tikz}
\usetikzlibrary{shapes.geometric, arrows, shadows, bayesnet}
\usepackage{xcolor}
\usepackage[colorlinks=true,linkcolor=BrickRed,urlcolor=Blue,citecolor=Blue,backref=page]{hyperref}
\usepackage[capitalize]{cleveref}

\usepackage{enumitem}
\usepackage{soul}
\usepackage{enumitem} 
\usepackage{doi}
\usepackage{thmtools,thm-restate}
\RequirePackage{amsmath}
\RequirePackage{amssymb}
\RequirePackage{mathtools}
\numberwithin{equation}{section}
\usepackage{times}
\usepackage[overload]{textcase}
\makeatletter
\let\ebcrlAistatsSect\@sect
\makeatother
\let\ebcrlTeXSubscript\sb

\crefname{figure}{Fig.}{Figs.}
\crefname{definition}{Defn.}{Defns.}
\crefname{corollary}{Cor.}{Cors.}
\crefname{proposition}{Prop.}{Props.}
\crefname{theorem}{Thm.}{Thms.}
\crefname{remark}{Remark}{Remarks}
\crefname{principle}{Principle}{Principles}
\crefname{lemma}{Lemma}{Lemmata}
\crefname{claim}{Claim}{Claims}
\crefname{table}{Tab.}{Tabs.}
\crefname{section}{\S}{\S\S}
\crefname{subsection}{\S}{\S\S}
\crefname{subsubsection}{\S}{\S\S}
\crefname{assumption}{Asm.}{Asms.}
\crefname{appendix}{Appx.}{Appx.}
\crefname{equation}{Eq.}{Eqs.}
\crefname{algorithm}{Alg.}{Algs.}

\renewcommand{\mathbf}[1]{{\bm{#1}}}
\newcommand{\ab}{\mathbf{a}}

\newcommand{\ub}{\mathbf{u}}

\newcommand{\xb}{\mathbf{x}}
\newcommand{\yb}{\mathbf{y}}
\newcommand{\zb}{\mathbf{z}}
\renewcommand{\sb}{\mathbf{s}}

\newcommand{\Ab}{\mathbf{A}}

\newcommand{\Db}{\mathbf{D}}

\newcommand{\Ib}{\mathbf{I}}
\newcommand{\Jb}{\mathbf{J}}

\newcommand{\Ob}{\mathbf{O}}

\newcommand{\Acal}{\mathcal{A}}

\newcommand{\Dcal}{\mathcal{D}}
\newcommand{\Ecal}{\mathcal{E}}

\newcommand{\Gcal}{\mathcal{G}}

\newcommand{\Ical}{\mathcal{I}}

\newcommand{\Ncal}{\mathcal{N}}

\newcommand{\RR}{\mathbb{R}} %

\newcommand*{\thetab}{{\bm{\theta}}}

\newcommand*{\epsilonb}{{\bm{\varepsilon}}}

\newcommand*{\gammab}{{\bm{\gamma}}}

\newcommand*{\phib}{{\bm{\phi}}}

\newcommand{\pa}{\mathrm{pa}}

\ifx\proof\undefined 
\RequirePackage{amsthm}
\fi

\ifx\BlackBox\undefined
\newcommand{\BlackBox}{\rule{1.5ex}{1.5ex}}  %
\fi

\ifx\QED\undefined
\def\QED{~\rule[-1pt]{5pt}{5pt}\par\medskip}
\fi

\ifx\proof\undefined
\newenvironment{proof}{\par\noindent{\bf Proof\ }}{\hfill\BlackBox\\[2mm]}
\fi
\ifx\proofsketch\undefined
\newenvironment{proofsketch}{\par\noindent{\bf Proof Sketch\ }}{\hfill\BlackBox\\[2mm]}
\fi

\ifx\theorem\undefined
\newtheorem{theorem}{Theorem}
\numberwithin{theorem}{section}
\fi
\ifx\property\undefined
\newtheorem{property}{Property}
\fi
\ifx\corollary\undefined
\newtheorem{corollary}[theorem]{Corollary}
\fi
\ifx\lemma\undefined
\newtheorem{lemma}[theorem]{Lemma}
\fi
\ifx\proposition\undefined
\newtheorem{proposition}[theorem]{Proposition}
\fi
\ifx\assum\undefined

\fi

\ifx\definition\undefined
\newtheorem{definition}[theorem]{Definition}
\fi

\ifx\remark\undefined
\newtheorem{remark}[theorem]{Remark}
\fi
\ifx\example\undefined
\newtheorem{example}[theorem]{Example}
\fi
\ifx\lemma\undefined
\newtheorem{lemma}[theorem]{Lemma}
\fi
\ifx\conjecture\undefined
\newtheorem{conjecture}[theorem]{Conjecture}
\fi
\ifx\fact\undefined
\newtheorem{fact}[theorem]{Fact}
\fi
\ifx\claim\undefined
\newtheorem{claim}[theorem]{Claim}
\fi
\ifx\assum\undefined

\fi

\newcommand\independent{\protect\mathpalette{\protect\independenT}{\perp}}
\def\independenT#1#2{\mathrel{\rlap{$#1#2$}\mkern2mu{#1#2}}}

\DeclareMathOperator*{\argmax}{arg\,max}
\DeclareMathOperator*{\argmin}{arg\,min}

  \usepackage{colortbl}

  \newcommand{\norm}[1]{\left\lVert#1\right\rVert}

\usepackage{setspace}

\usepackage[toc,page,header]{appendix}
\usepackage{minitoc}

\AtBeginDocument{%
  
}

\newcommand{\ELBO}{\mathrm{ELBO}}

\newcommand{\R}{\mathbb{R}}

\newcommand{\N}{\mathbb{N}}

\newcommand{\ba}{\bm{a}}

\newcommand{\bu}{\bm{u}}

\newcommand{\bx}{\bm{x}}
\newcommand{\by}{\bm{y}}
\newcommand{\bz}{\bm{z}}

\newcommand{\iid}{\stackrel{\mathrm{iid}}{\sim}}
\newcommand{\rmp}{\mathrm{p}}
\newcommand{\rmq}{\mathrm{q}}

\newcommand{\cA}{\mathcal{A}}

\newcommand{\cD}{\mathcal{D}}
\newcommand{\cE}{\mathcal{E}}

\newcommand{\cG}{\mathcal{G}}

\newcommand{\cI}{\mathcal{I}}

\newcommand{\cM}{\mathcal{M}}
\newcommand{\cN}{\mathcal{N}}
\newcommand{\cO}{\mathcal{O}}
\newcommand{\cP}{\mathcal{P}}

\newcommand{\cS}{{\mathcal{S}}}

\newcommand{\bA}{\bm{A}}

\newcommand{\bC}{\bm{C}}
\newcommand{\bD}{\bm{D}}

\newcommand{\bF}{\bm{F}}

\newcommand{\bL}{\bm{L}}
\newcommand{\bM}{\bm{M}}

\newcommand{\bO}{\bm{O}}
\newcommand{\bP}{\bm{P}}

\newcommand{\bU}{\bm{U}}
\newcommand{\bV}{\bm{V}}

\newcommand{\E}[1]{\mathbb{E}\left[#1\right]}

\newcommand{\EE}[2]{\mathbb{E}_{#1}\left[#2\right]}

\newcommand{\KL}[2]{D_{\textsc{KL}}\left( #1 \parallel #2\right)}

\newcommand{\rmg}{\textrm{g}}

\newcommand{\mbF}{\mathbf{F}}

\newcommand{\dd}{\mathrm{d}}

\newcommand{\ch}{\textrm{ch}}
\newcommand{\zetab}{{\bm\zeta}}

\makeatletter
\let\@sect\ebcrlAistatsSect
\let\ebcrlUrlMathSetup\Url@MathSetup
\def\Url@MathSetup{\let\sb\ebcrlTeXSubscript\ebcrlUrlMathSetup}
\makeatother
\usepackage[most]{tcolorbox}

\newcommand{\BW}[1]{}

\newcommand{\jvk}[1]{}
\ifnum\statePaper=0
  \renewcommand{\BW}[1]{}
  \renewcommand{\jvk}[1]{}
  \hypersetup{pdfauthor={},pdfsubject={},pdfkeywords={}}
  \ifdefined\pdfsuppressptexinfo\pdfsuppressptexinfo=15\relax\fi
\fi
\hypersetup{pdfauthor={Bohan Wu, Julius von K\"ugelgen, David M. Blei},pdftitle={Multi-Domain Empirical Bayes for Linearly-Mixed Causal Representations}}
\colorlet{BLUE}{black} %

\begin{document}
\twocolumn[
\aistatstitle{Multi-Domain Empirical Bayes for Linearly-Mixed Causal Representations}
\aistatsauthor{Bohan Wu \And Julius von K\"ugelgen \And David M.\ Blei}
\aistatsaddress{Columbia University \And ETH Z\"urich \And Columbia University}
]
\begin{abstract}
\vspace{-.5em}
Causal representation learning (CRL) aims to learn low-dimensional causal latent variables from high-dimensional observations. While identifiability
has been extensively studied for CRL, estimation has been less explored. We use empirical Bayes (EB) to estimate causal representations from multiple domains whose differences reflect interventions in a shared causal model. Multi-domain CRL poses a simultaneous inference problem suited to EB.
Here, we propose an EB
$f$-modeling algorithm that improves the quality of learned causal variables by exploiting invariant structure within and across domains. Specifically, we consider linear measurement models %
with interventional priors induced by an acyclic latent causal model with known graph and intervention targets%
. 
We develop an EM-style algorithm based on causally structured score matching.  
In experiments on synthetic and semi-synthetic single-cell data,  our proposed method achieves more accurate estimation than alternative methods and baselines. 
\end{abstract}

\section{Introduction}
\looseness-1 
Causal representation learning (CRL) seeks to recover low-dimensional latent causal variables from high-dimensional realizations~\citep{scholkopf2002learning,moran2026towards,acart2026learning}. 
Since this task is highly challenging in general~\citep{hyvarinen1999nonlinear,locatello2019challenging}, a common approach is to leverage heterogeneous (non-i.i.d.) data as a learning signal~\citep{von2024identifiable,hyvarinen2019nonlinear}.
A central idea is that  invariances under heterogeneity can reveal causal information~\citep{yao2025unifying}. 
A prominent example
is grouped data from different environments or domains.
From a causal perspective, such \emph{multi-domain} data is typically viewed as arising from sparse interventions in an underlying causal model~\citep{peters2016causal,perry2022causal}, while the remaining causal mechanisms and the measurement process are shared across domains.

\looseness-1 
Most prior work on CRL, including in the multi-domain setting, 
has focused on the question of \emph{identifiability}~\citep{squires2023linear,Buchholz2023,von2023nonparametric,ahuja2023interventional,varici2025score,zhang2024identifiability,zhang2024causal,wendong2023causal,jin2024learning}: under what conditions and up to what ambiguities
can the latent causal variables provably be recovered at the population level, i.e., in the
infinite-data regime?
In this paper, we focus on the \textit{estimation problem}. How can we learn causal representations from finite samples of noisy measurements?

CRL naturally amounts to a \emph{simultaneous inference} problem: each observation has its own local latent variable, yet
these variables follow a shared distribution induced by the underlying causal model. In the multi-domain CRL setting, such information sharing occurs within and across domains, since the measurement process and parts of the causal model are assumed to remain invariant.

\textit{Empirical Bayes} (EB) offers a solution to simultaneous inference of many local latent variables~\citep{robbins1956empirical,efron2019bayes,Ignatiadis2025EB,wu2025bayesian}. The canonical EB setup assumes all local latent variables
are i.i.d.\ according to a shared and unknown prior, and then solves the problem of simultaneous inference
via shrinkage~\citep{james1961estimation}. Multi-domain CRL provides a novel version of this problem, where the local variables are the latent causal variables,
and the domains exhibit structured heterogeneity through sparse interventions. We develop EB methods for this CRL setting.

\textbf{Overview and contributions.} 
In this work, we adopt a probabilistic perspective on CRL with noisy measurements and focus on finite-sample estimation via EB.
We first develop the connection between CRL and EB
at the conceptual level of modeling principles~(\cref{sec:problem_setting}), and then study a concrete class of noisy multi-domain latent causal models~(\cref{sec:model_class}).
Specifically, we consider linear measurement models~(\cref{sec:linear_measurement_model}) with interventional priors induced by an acyclic latent SCM with known causal graph and intervention targets~(\cref{sec:priors}).

\looseness-1 
We then explore EB estimation for this model class.
Under an orthogonality condition on the mixing matrix, projecting observations yields a normal means model, which facilitates an EM-style $f$-modeling EB approach~(\cref{sec:EB_estimation_inference}), wherein posterior denoising can be expressed through Tweedie's formulas in terms of the score of the marginal distribution. We leverage the causal graph structure to parameterize and efficiently estimate this score  across domains~(\cref{sec:causal_score_matching}), yielding an iterative algorithm (\cref{alg:iterative-SM}) that alternates score estimation, Tweedie updates for denoising latents, and
EM updates for the mixing matrix and noise variance.
Through empirical studies on synthetic and semi-synthetic data based on a causal protein-signaling network~\citep{sachs2005causal}~(\cref{sec:experiments}), we demonstrate that CRL $f$-modeling improves latent recovery relative to natural baselines and provides stable performance across domains.

\textbf{Notation.} \looseness-1 
We write column vectors in bold lowercase (e.g., $\bx \in \R^d$), with $x_i$ denoting the $i$\textsuperscript{th} entry, and matrices in bold uppercase (e.g., $\bA \in \R^{d \times k}$), with $\bA_i$ denoting the $i$\textsuperscript{th} row and $\bA_{ij}$ the $(i,j)$\textsuperscript{th} entry. 
Let $\cO^{d \times k}$ denote the Stiefel manifold of $d \times k$ matrices with orthonormal columns. For $n \in \N$, define $[n] := \{1, \dots, n\}$. 
  The set $\cP(\Omega)$ stands for the set of
  probability measures over $\Omega$. In a directed graph, $\pa(j)$
  and 
$\ch(j)$ denote the parents
and
children of node~$j$.
Partial derivatives w.r.t.\ the $j$\textsuperscript{th} argument are denoted by $\partial_j$ and gradients by $\nabla$. If $\bx \in \R^d$,  $\bx \odot \bx$ denotes the vector of element-wise squares, i.e., $\bx \odot \bx = (x_1^2,\ldots,x_d^2)^\top$. 

\section{Causal Empirical Bayes}
\label{sec:problem_setting}
\looseness-1 
Consider $M$ domains, each characterized by an observed action or perturbation label ${\ab_e\in\Acal:=\{\ab_1, ..., \ab_M\}\subseteq
\RR^{d_A}}$.
For each $\ab_e\in\Acal$, we have access to an i.i.d.\ sample of realizations $\xb\in\RR^{d_X}$ from an unknown population distribution $\rmp^\star(\xb\mid\ab_e)$. That is, we observe data
\begin{equation}
\begin{aligned}
\label{eq:data}         
    \Dcal&=\left(\left(\left(\xb_{ei}\right)_{i=1}^{N_e}, \ab_e
    \right)_{e=1}^M\right), 
    \\
    \text{where} \qquad 
    \left(\xb_{ei}\right)_{i=1}^{N_e}&\iid \rmp^\star(\xb\mid\ab_e).
\end{aligned}
\end{equation}
The distributions $
\{\rmp^\star(\bx\mid\ab)\}_{\ab\in\Acal}$ can be thought of as \emph{interventional} distributions, 
in the sense that each $\ab_e$ contains (partial) information on which intervention or experiment was performed in domain~$e$. 

\looseness-1 
To model $\rmp^\star(\xb\mid\ab)$, we use a latent variable model with local latents $\zb\in\RR^{d_Z}$, one for each~$\xb$. 
In the context of single-cell measurements $\xb$, these latent variables could represent, e.g., gene programs or clusters of proteins. 
Importantly, we assume that  $\zb$ fully mediates the effect of $\ab$ on~$\xb$, so that the relationship between $\zb$ and $\xb$ (e.g., the measurement process) is invariant across domains while only the distribution of $\zb$ may change, see~\Cref{fig:graphical_model}. 
 We formalize this as follows.
\begin{assumption}[Causal empirical Bayes]
\label{assump:causal-eb}
For all $\ab\in\Acal$, there exists an interventional prior $\rmp^\star(\bz\mid\ab)$ s.t.\ the true interventional marginal $\rmp^\star(\bx\mid\ab)$ can be expressed~as
\begin{equation} \label{eq:causal-eb}
        \rmp^\star(\bx\mid\ab) = \int_{\RR^{d_Z}} \, \rmp^\star(\bx\mid\bz)\,\rmp^\star(\bz\mid\ab)\, d\bz. 
\end{equation}
\end{assumption}
\looseness-1 We refer to~\cref{assump:causal-eb} as the \emph{causal empirical Bayes assumption}, as it connects the Bayesian latent variable model  on the RHS to the true causal data-generating process on the LHS. Accordingly, we refer to $\bz$ as the \emph{latent causal variables}. 
\Cref{eq:causal-eb} admits two complementary interpretations.
    Read from right to left, it implies that the probabilistic model $\bz \sim \rmp^\star(\bz\mid\ab)$ and $\bx \sim \rmp^\star(\bx\mid\bz)$ is a \emph{well-specified} generative model for the true interventional distribution $\rmp^\star(\bx\mid\ab)$.
    Read from left to right, on the other hand, it suggests an \emph{empirical Bayes} interpretation: given the true marginal $\rmp^\star(\bx\mid\ab)$ and the likelihood $\rmp^\star(\bx\mid\bz)$, we may infer the true prior $\rmp^\star(\bz\mid\ab)$. In this view, the prior can be learned from the data using a $\rmg$-modeling approach \citep{Efron2015deconvolution} or indirectly using $f$-modeling \citep{robbins1956empirical,Efron2011}, 
   see~\cref{sec:EB_estimation_inference} for details.

Our target of inference is
the posterior of the latent causal variables: 
\begin{equation} \label{target-posterior}
        \rmp^\star \left(\bz\mid\bx_{ei}, \ba_e \right) \propto \rmp^\star(\bz\mid\ab_e)\, \rmp^\star(\bx_{ei}\mid\bz). 
\end{equation}
\looseness-1 Here, both the measurement model $\rmp^\star(\bx\mid\bz)$ and the domain-specific priors $\rmp^\star(\bz\mid\ab_e)$ are unknown quantities. 

\section{CRL Model Class}
\label{sec:model_class}
We now describe the assumptions we place on the measurement model~(\cref{sec:linear_measurement_model}) and domain-specific priors~(\cref{sec:priors}).

\begin{figure}[t]
    \newcommand{\xshift}{6em}
    \newcommand{\yshift}{4em}
    \centering
    \begin{tikzpicture}
        \centering
        \node(a)[obs]{$\ab_e$};
        \node(z)[latent, xshift=\xshift]{$\zb_{ei}$};
        \node(x)[obs, xshift=2*\xshift]{$\xb_{ei}$};
        \edge[thick, -stealth]{a}{z};
        \edge[thick, -stealth]{z}{x};
        \plate[inner sep=0.5em]{i} {(z) (x)} {$N_e$};
        \plate[]{e} {(a) (z) (x) (i)} {$M$};
 \end{tikzpicture}
    \caption{\textbf{Assumed graphical model.} Shaded nodes are observed, white nodes are latent. The inner plate is over realizations $i=1, ...,N_e$, the outer over domains $e=1,...,M$.}
    \label{fig:graphical_model}
\end{figure}

\subsection{Linear Gaussian Measurement Model}
\label{sec:linear_measurement_model}
We focus on linear measurement models with additive, zero-mean, isotropic Gaussian noise.
\begin{assumption}[Linear Gaussian measurement model]
\label{ass:linear_measurement_model}
For all $e\!\in\![M], i \in [N_e]$, 
the generative process is
\begin{equation}
\label{eq:multi_domain_linear_measurements}
\begin{aligned}
\zb_{ei} &\sim \rmp^\star(\zb \mid \ab_e), 
\\
\xb_{ei} &= \Ab^\star \zb_{ei} + \epsilonb_{ei},
\qquad
\epsilonb_{ei}\sim\Ncal\big(\bm 0, \sigma^{2\star}\Ib_{d_X}\big),
\end{aligned}
\end{equation}
 for some unknown true mixing or loading matrix ${\Ab^\star \in \RR^{d_X \times d_Z}}$ and  unknown true noise  variance $\sigma^{2\star}$.
\end{assumption}
\looseness-1 
Consider a single domain ($M\!=\!1$) with equal latent and observed dimensionality ($d_X\!=\!d_Z$) and no mixing ($\Ab\!=\!\Ib_{d_X}$). In this setting, the model in~\eqref{eq:multi_domain_linear_measurements} corresponds to (a multi-variate version of) the classical empirical Bayes normal means model \citep{robbins1956empirical}, which aims to recover the unknown mean~$\zb_{i}\!=\!\mathbb{E}[\xb_i]$ for each observation~$\xb_{i}$, under the assumption that the means are drawn from a single shared prior ${\zb_{i}\!\sim\!\rmg}$.  The single-domain setting with $d_X\!>\!d_Z$ and \textit{unknown} $\Ab$ constitutes a generalization of this classical setting, which is related to probabilistic PCA~\citep{Tipping1999PPCA}.\footnote{\looseness-1 Without additional assumptions on $\Ab^\star$ and $\rmp^\star(\zb)$, however, the model is not identifiable and recovering components with maximal explained variance may not correspond to recovering the~$z_j$'s.
} 

For multiple domains ($M>1$), \cref{ass:linear_measurement_model} generalizes the classical setting, where each domain has a possibly different prior $\zb_{ei}\!\sim\!\rmg_e\!:=\!\rmp^\star(\zb\mid\ab_e)$.
That is, an empirical Bayes setup plays out in each domain. 
However, 
there is not only information sharing \textit{within} domains, through the shared priors $\rmg_e$, but also \textit{across} domains, through the shared measurement model and shared causal structure among
the domain-specific priors.
This raises the following central questions, which we seek to answer in the present work.
\vspace{-0.25em}
\begin{tcolorbox}[boxsep=-0.25em]
\textit{How can the assumed multi-domain structure be leveraged to accurately estimate the shared measurement model and the latent variables? Can empirical Bayes estimation be generalized to the model from~\eqref{eq:multi_domain_linear_measurements}?}
\end{tcolorbox}
\vspace{-0.25em}
Rather than study these questions for arbitrary nonparametric priors $\rmg_e$, we place additional structure on the domain-specific priors to incorporate the available information in the form of the observed action labels~$\ab_e$.

\subsection{Interventional Structural Priors}
\label{sec:priors}
To set the
interventional priors 
$\{\rmp^\star(\bz\mid\ba_e)\}_{\ab_e \in \cA}$,
we adopt the structural causal model framework~\citep[SCM;][]{Pearl2009, Peters2017} to describe the causal relations among latent variables. 

\begin{definition}[Acyclic SCM] \label{def-SCM}
Let $\bz = (z_1, \ldots, z_d)$ denote endogenous (causally determined) variables, and let $\bu = (u_1, \ldots, u_d)$ denote exogenous (noise) variables. A \emph{structural causal model} (SCM) $\cM = (\mathbf{F}, \rho_\ub)$ over $\bz$ consists of:
(i) a collection of $d$ structural assignments:
\begin{equation}
\label{eq:structural_equations}
\bF = \big\{ z_j := f_j\big(\bz_{\pa(j)}, u_j \big) \big\}_{j=1}^d,
\end{equation}
where each $f_j$ is a measurable function mapping the values of the parents, or direct causes, $\zb_{\pa(j)}$ of $z_j$, and noise $u_j$ to the value of $z_j$; and
(ii) a distribution over the exogenous variables: $\bu \sim \rho_{\bu}$.
\looseness-1 
The \emph{causal graph $\cG$ induced by~$\cM$}, given by the directed graph with nodes $\{1,\ldots,d\}$ 
and edges $\Ecal=\{(k,j)\mid k\in\pa(j)\}$,
is assumed acyclic. The \emph{distribution $\rmp_{\cM}(\zb)$ induced by $\cM$} is given by the pushforward of $\rho_\ub$ under the structural assignment map~$\mathbf{F}$.
\end{definition}
Throughout, we additionally assume the following.\footnote{\Cref{assumption:causal-sufficiency} is slightly stronger than
causal sufficiency, which only implies independence of the noise variables. Here, we also assume identical marginal distributions for convenience.} 

\begin{assumption}[IID noise]
\label{assumption:causal-sufficiency}
The exogenous noise variables are jointly independent and identically distributed. That is,  $u_1, \ldots, u_d \iid \rho$ for some probability measure $\rho$.
\end{assumption}
\Cref{assumption:causal-sufficiency} rules out hidden confounding and guarantees that the induced distribution is Markov w.r.t.\ the induced causal graph $\cG$, 
i.e., that it obeys the following factorization, 
\begin{equation}
    \label{eq:causal_Markov_factorisation}
    \textstyle 
    \rmp_\cM(\zb)=\prod_{j=1}^{d} \rmp_\cM\big(z_j\mid\zb_{\pa(j)}\big).
\end{equation}
The structural assignments in~\cref{eq:structural_equations} represent independent mechanisms which remain invariant if other parts of the model change. Interventions are modelled by replacing a subset of these mechanisms in the original model.
Here, we consider interventions (deterministic or stochastic, perfect or imperfect) which do not introduce new parents. 
\begin{definition}[Interventions]
\label{def:interventions}
An \emph{intervention} on targets $\Ical\subseteq[d]$ in an acyclic SCM $\cM = (\mbF, \rho)$ replaces a subset  of assignments
in $\bF$ with new assignments $$\{z_j := h_j(\bz_{\pa(j)}, \tilde u_j)\}_{j\in\Ical}$$ such that~\cref{assumption:causal-sufficiency} holds for the new noise variables $\tilde\ub_\Ical\cup\ub_{[d]\setminus\Ical}$. This yields a modified SCM
$\cM'= (\mbF', \rho)$,
where $\mbF'$ matches $\mbF$ except that $f_j$ is replaced by~$h_j$ for all $j\in\Ical$.  The resulting \emph{interventional distribution} and \emph{post-intervention graph} $\cG'$ are those induced by $\cM'$.
\end{definition}

We now specify the true prior structure
by linking the observed action labels to interventions in an unknown underlying SCM. 
Specifically, we consider the case of binary $\ab_e\in\{0,1\}^{d_Z}$ indicating the intervention targets in domain~$e$, see~\cref{sec:discussion} for a discussion of alternative choices.
\begin{assumption}[Interventional priors with known targets and causal graph] 
\label{ass:SCM-prior}
\looseness-1 There exists an acyclic SCM $\cM^\star$ over~$\zb$ satisfying  
\Cref{assumption:causal-sufficiency} such that for all $e \in [M]$ the true prior $\rmp^\star(\bz\mid\ba_e)$ is an interventional distribution induced by an intervention on targets $\Ical(\ab_e)=\{j\mid a_{ej}=1\}$ in $\cM^\star$ and $$(\bz_{ei})_{i \in [N_e]}\iid \rmp^\star(\bz\mid\ba_e).$$  The causal graph~$\cG^\star
$ induced by $\cM^\star$ is assumed known.
\end{assumption}
Under \Cref{ass:SCM-prior}, for all $e \in [M]$, $i \in [N_e]$, and $j \in [d_Z]$, the mechanism that gives rise to $z_{eij}$ 
is given by 
\begin{equation} 
\label{intervention mechanism}
    z_{eij} :=
    \begin{cases}
        f_j\big(\bz_{ei\pa(j)},\, u_{eij}\big), & \text{if } a_{ej} = 0, \\
        h_j\big(\bz_{ei\pa(j)},\, \tilde u_{eij}\big), & \text{if } a_{ej} = 1,
    \end{cases} 
\end{equation}
where $f_j$ is the baseline and $h_j$ the interventional mechanism when node $j$ is targeted by $\ab_e$.
This  extends the classical empirical Bayes setting, which would assume  $z_{eij} \iid \rmg$ 
for a single, shared, learnable prior $\rmg$ \citep{efron2019bayes,Ignatiadis2025EB}.

\section{Estimation and Inference}
\label{sec:EB_estimation_inference}
Causal representation learning (CRL) is naturally a simultaneous inference problem, as one of its goals is to recover the causal variables $\{\bz_{ei}\}_{i\in[N_e],\,e\in[M]}$ 
from the available multi-domain data $\cD$ in \Cref{eq:data}. This amounts to $\sum_{e=1}^M N_e$ simultaneous
inference problems of computing the posteriors 
$\rmp^\star(\bz_{ei} \mid \bx_{ei}, \ba_e)$ in \Cref{target-posterior},
thus creating an opportunity to apply the empirical Bayes (EB) methodology.

The main challenge for posterior inference is that the measurement model parameters~$\bA^\star$ and $\sigma^{2\star}$ and the priors $\rmp^\star(\zb\mid\ab_e)$ are unknown.\footnote{ 
According to~\cref{ass:SCM-prior} the causal graph underlying $\rmp^\star(\bz \mid \ba)$ is known and the intervention targets in domain $e$ are given by $\ab_e$. 
However, the shared
base causal mechanisms $\{f_j\}_{j \in [d_Z]}$ and the intervened mechanisms $\{h_j\}_{j \in [d_Z]}$ in \Cref{intervention mechanism} are unknown.} EB methods infer these quantities via~\Cref{assump:causal-eb} and
fall into two categories, {$\rmg$-modeling} and {$f$-modeling}, depending on how they 
handle the prior~\citep{Ignatiadis2025EB}.

\looseness-1 In $\rmg$-modeling, the prior (typically denoted~$\rmg$) is modeled explicitly
and estimated jointly with the other unknowns by approximately optimizing the marginal likelihood of~$\cD$. The fitted model is then used for posterior inference.
We further develop and discuss the $\rmg$-modeling interpretation of CRL,
as well as its connections to variational inference and existing CRL algorithms in~\cref{sec:EB_g_modeling}.

In $f$-modeling, the learning of the prior remains implicit. Instead, posterior quantities are estimated directly by modeling the marginal density (typically denoted $f$) of the observed data~\citep{Efron2014}. This is often done via the classical Robbins-Tweedie formula, which uses the score of the marginal likelihood to approximate the posterior mean and covariance of the local latent variables~\citep{Eddington1940,Efron2011,meng2021estimating}. 

In the present work, we focus on $f$-modeling via Tweedie's formula and adapt this method to our multi-domain CRL setting. 
As we will show, $f$-modeling provides a direct route to obtaining the posterior summaries of $\rmp^\star(\bz \mid \bx, \ba)$ by estimating the data distribution $\rmp^\star(\bx \mid \ba)$ in \Cref{eq:causal-eb}, without ever explicitly learning the domain-specific priors $\rmp^\star(\bz \mid \ba_e)$.
Specifically, we develop an expectation maximization~(EM)~\citep{dempster1977maximum} algorithm for CRL $f$-modeling, which leverages a causally-structured score estimator.
We describe the algorithm's main steps below and score estimation in~\cref{sec:causal_score_matching}.

\textbf{High-level EM procedure.}
Our EM $f$-modeling algorithm for model~\eqref{eq:multi_domain_linear_measurements} iterates between estimating the first and second moments of the latent causal variables~$\bz_{ei}$ (E-step) and estimating the measurement model parameters, i.e., the mixing matrix~$\bA$ and the noise variance $\sigma^2$ (M-step).
Each iteration runs in two steps: 
\begin{enumerate}[topsep=0em,itemsep=0em]
    \item[\textbf{E:}] Fix the current estimate $\widehat{\bA}$ and $\widehat \sigma^2$
    and compute the posterior mean $\widehat{\bz_{ei}}$ and element-wise  second moments $\widehat{\bz_{ei}^2}$ using Tweedie’s formula \citep{Eddington1940}, combined with score matching \citep{hyvarinen2005estimation} (see~\cref{sec:causal_score_matching} for details on score estimation).
    \item[\textbf{M:}] Update $\widehat{\bA}$ and  $\widehat \sigma^2$ by (approximately) maximizing the likelihood given the current posterior summaries~$\widehat{\bz_{ei}}$ and $\widehat{\bz_{ei}^2}$ from the E-step.
\end{enumerate}
The full procedure, which alternates these updates for a fixed number of iterations, is summarized in~\Cref{alg:iterative-SM}. We will also refer to the algorithm as \emph{CRL $f$-modeling}. 
Whereas the M-step is relatively straight-forward, the E-step is more involved; this is also where our causal assumptions enter.

\begin{algorithm}[tb]
\caption{EM $f$-modeling with causal score matching}
\label{alg:iterative-SM}
\begin{algorithmic}
\REQUIRE Data $\Dcal=(((\bx_{ei})_{i=1}^{N_e},\,\ba_e))_{e=1}^M$, causal graph $\cG=([d_Z],\cE)$, iterations $T$, damping factor $\eta\in[0,1]$.

\STATE Initialize $\widehat{\Ab}\!=\!\widehat{\Ob}\widehat{\Db}$ with $\widehat{\Ob}\!\in\!\cO^{d_X\times d_Z}$ and $\widehat{\Db}\!\in\!\R^{d_Z\times d_Z}$ diagonal, score components $\{\widehat s_j\}_{j\in[d_Z]}$, and $\widehat{\sigma}^2>0$.

\FOR{$t=1$ \textbf{to} $T$}
    \STATE Compute $\yb_{ei}\gets \widehat{\Ob}^\top \xb_{ei}$ for all $e\in[M]$ and $i\in[N_e]$.%
    \STATE Update $\widehat s$ via causal score matching: \eqref{eq:score-matching-j} and~\eqref{score-matching}.
    \STATE Update $\widehat{\zb}_{ei}$ and $\widehat{\zb}^2_{ei}$ using the first- and second-order Tweedie formulas in~\eqref{eq:Tweedie} and \eqref{eq:second-order-Tweedie}.
    \STATE Update $\widehat{\Ab}=\widehat{\Ob}\widehat{\Db}$ via the analytical steps in~\eqref{A-update}.
    \STATE Update $\widehat{\sigma}^2$ using the MLE in~\eqref{mle-var}.
\ENDFOR
\STATE \textbf{return} $\{\widehat{\zb}_{ei}\}_{e\in[M],\,i\in[N_e]}$, $\widehat{\Ab}$, $\widehat{\sigma}^2$.
\end{algorithmic}
\end{algorithm}

\textbf{Tweedie's formula.}
The focal point of traditional empirical Bayes theory is the normal means model with isotropic noise~\citep{robbins1956empirical,Efron2014,Efron2015deconvolution,efron2019bayes,Soloff2021},
\begin{equation}
    \xb_i=\bm\theta_i+\bm\zeta_i, \qquad \bm\zeta_i\iid\Ncal(\bm0,\sigma^2\Ib_{d_X}).
\end{equation}
A central result is Tweedie’s formula \citep[][Thm.~6.2]{Ignatiadis2025EB}, which states that the optimal estimator $t^\star$ minimizing the posterior risk $\E{\|\bm\theta-t(\bx_{i})\|_2^2}$ for any prior on $\bm\theta_{i}$ is given by
\begin{equation*}
    t^\star(\bx_{i}) := \E{\bm\theta_i \mid \bx_{i}} = \bx_{i} + \sigma^2 \nabla \log {f}(\bx_{i}),
\end{equation*}
\looseness-1
where ${f}$ denotes the marginal density of $\bx$.

\looseness-1 
By applying Tweedie's formula to model~\eqref{eq:multi_domain_linear_measurements}, we can recover the posterior mean of $\bA \bz$ given $\bx$. However, this does not lead to the posterior mean of $\bz$ as $\bA$ need not be invertible. To solve this problem, we impose an additional assumption on the class of mixing matrices $\bA$ that allows us to reduce the problem to a normal means model. 

\textbf{Reduction to normal means model.}
To model the causal variable $\bz$ via Tweedie's formula, we impose the following orthogonality assumption on~$\bA^\star$.
\begin{assumption} 
\label{ass:orthonormal-A}
\looseness-1
The true mixing matrix $\bA^\star$ satisfies $(\bA^\star)^\top \bA^\star = (\bD^\star)^2$ for some diagonal~$\bD^\star \in \R^{d_Z \times d_Z}$.
\end{assumption}
\looseness-1 \Cref{ass:orthonormal-A} states
that the columns of $\bA^\star$ are orthogonal with possibly different norms, similar to assumptions exploited in independent mechanism analysis~\citep{Gresele2021,reizinger2022embrace} or principal component flows~\citep{cunningham2022principal}. 
\Cref{ass:orthonormal-A} allows us to parametrize~$\bA$ as $\bA:= \bO \bD$, where $\bO \in \cO^{d_X \times d_Z}$ is orthonormal and $\bD\in\RR^{d_Z\times d_Z}$ is diagonal. 

Let 
$\by_{ei}\!:=\!\bO^\top \bx_{ei}$. Since $\bO^\top \bO\!=\!\Ib_{d_Z}$, multiplying both sides of the linear measurement model in~\eqref{eq:multi_domain_linear_measurements}  by $\Ob^\top$ then yields the following (multi-domain) normal means model,
\begin{equation} 
\label{eq:normal-means}
    \by_{ei} = \bD \bz_{ei} + \zetab_{ei}, \qquad \zetab_{ei} \iid \cN(0, \sigma^2 \Ib_{d_Z}).
\end{equation}

\textbf{Updating the latent variables.}
According to Tweedie's formula applied to~\cref{eq:normal-means}, the optimal estimator for the means $\Db\zb_{ei}$ is given by
\begin{equation}
\label{eq:tweedie_reduction}
    \E{\bD \bz_{ei} \mid \by_{ei}} = \by_{ei} + \sigma^2 \nabla \log {f}_{\ba_e}(\by_{ei}),
\end{equation}
\looseness-1
where ${f}_{\ba_e}$ is the true marginal density of $\by$ in domain~$e$, defined as the push-forward of $\rmp^\star(\zb\mid\ab_e)$ via~\eqref{eq:normal-means}.

For fixed $\widehat\Ab=\widehat\Ob\widehat\Db$ and $\widehat\sigma^2$, we therefore only need estimates of the scores~$\nabla \log {f}_{\ba_e}(\by_{ei})$ to estimate the posterior mean of $\zb_{ei}$ via~\eqref{eq:tweedie_reduction}. In~\cref{sec:causal_score_matching}, we discuss how we estimate these scores using a flexible function class to obtain an estimate $\widehat s(\yb,\ba)$ via score matching~\citep{hyvarinen2005estimation}.
We then update the empirical Bayes estimates of all local latent variables $\bz_{ei}$ using a damped version of Tweedie’s formula: 
\begin{equation} \label{eq:Tweedie}
    \widehat{\bz_{ei}} = \widehat \bD^{-1} \left(\by_{ei} + \eta \widehat \sigma^2 \widehat{s}(\by_{ei}, \ba_e) \right), 
\end{equation}
where $\eta \in [0,1]$ is the damping factor;  $\eta = 0$ corresponds to no shrinkage and $\eta = 1$ to the classical Tweedie's formula. 

For the M-step, we also need the posterior second moments, defined as the element-wise squares ${\mathbb{E}[\bz_{ei} \odot \bz_{ei}\mid\bx_{ei},\ba_e]}$. To approximate them, we use %
estimates ${\widehat{\bz_{ei}^2} \in \R^{d_Z}}$ obtained from the second-order Tweedie formula~\citep{Efron2011,song2019generative}:
\begin{equation}\label{eq:second-order-Tweedie}
\widehat{z_{eij}^2} \!=\! \widehat{z_{eij}}^2\!+\! \widehat \sigma^2 \widehat \bD_{jj}^{-2} \!+\! \widehat \sigma^4 \widehat \bD_{jj}^{-2}\partial_j \widehat s_j(y_j, \yb_{\pa(j)}, \ba_e),
\end{equation}
where the component-wise score estimates $\widehat{s}_{jj}$ for $j\in[d_Z]$ are described in further detail in~\cref{sec:causal_score_matching}.

\textbf{Updating the measurement model.}
Given our estimates of the posterior first and second moments of each $\zb_{ei}$, in the M-step we update the parameters $(\bA,\sigma^2)$ to maximize the data log-likelihood,
resulting in
\begin{equation} \label{M-step}
\widehat \bA, \widehat \sigma^2
\in \argmax_{\bA = \bO \bD, \sigma^2 \geq 0} \sum_{e = 1}^M \sum_{i = 1}^{N_e}\EE{\widehat \rmq_{ei} \left(\bz\right)}{\log \rmp_{\bA, \sigma^2}\left(\bx_{ei} \mid \bz\right)}, 
\end{equation}
where $\widehat \rmq_{ei}$ is the posterior distribution that we approximated via~\eqref{eq:Tweedie}
and~\eqref{eq:second-order-Tweedie} in the E-step. The M-step objective also corresponds to the first (energy) term in the ELBO, see~\cref{sec:EB_g_modeling}.

Since $\log \rmp_{\bA,\sigma^2}(\bx_{ei}\mid \bz_{ei})$ is quadratic in $\bz_{ei}$ under the Gaussian measurement model (\cref{ass:linear_measurement_model}), this objective depends on $\widehat \rmq_{ei}$ only through its first and second moments. Thus, we replace $\EE{\widehat \rmq_{ei}}{\bz_{ei}}$ and $\EE{\widehat \rmq_{ei}}{\bz_{ei}\odot \bz_{ei}}$ by the empirical Bayes estimates $\widehat{\bz_{ei}}$ and~$\widehat{\bz_{ei}^2}$ to obtain objectives for $\bA$ and $\sigma^2$, see \Cref{app-alg-details} for the detailed derivations. 

\paragraph{Updating $\widehat\Ab$.} Given the estimates $\widehat{\bz_{ei}}$ and $\widehat{\bz_{ei}^2}$, we learn~$\bA$ by the maximum likelihood estimator (MLE), 
\begin{equation} \label{mle-A}
\hspace{-.7em}\widehat \bA  \in \arg\!\min_{\bA = \bO \bD} \sum_{e = 1}^M \sum_{i = 1}^{N_e} \left[ \sum_{j =1}^{d_Z}  \bD_{jj}^2 \widehat{ z^2_{eij}} - 2  \bx_{ei}^\top\bA \widehat{\bz_{ei}}\right]\!.
\end{equation}
The solution $\widehat \bA = \widehat \bO \widehat \bD$ to~\eqref{mle-A} can be computed analytically in the following steps:
\begin{align}
\bM &\gets \sum_{e = 1}^M \sum_{i=1}^{N_e} \bx_{ei}\widehat{\bz_{ei}}^\top,\qquad \left[\bU, \bm\Sigma, \bV\right] \gets \textsc{SVD}(\bM),
\nonumber
\\
\label{A-update}
\widehat \bO &\gets \bU \bV^\top, \qquad \widehat \bD_{jj}^2 \gets \frac{\bm\Sigma_{jj}}{\sum_{e = 1}^M \sum_{i=1}^{N_e} \widehat{z^2_{eij}}},
\end{align}
\paragraph{Updating $\widehat\sigma^2$.} 
Finally, we update $\sigma^2$ with the MLE: 
\begin{equation} \label{mle-var}
\widehat \sigma^2 =\frac{\sum_{e = 1}^M \sum_{i = 1}^{N_e} \left[ \sum_{j =1}^{d_Z}  \widehat \bD_{jj}^2 \widehat{ z^2_{eij}} - 2  \bx_{ei}^\top \widehat \bA \widehat{\bz_{ei}} + \|\bx_{ei}\|^2 \right]}{d_X \sum_{e = 1}^M N_e}. 
\end{equation}

Together,   \Cref{eq:Tweedie}--\eqref{eq:second-order-Tweedie} form the E-steps and \Cref{mle-A}--\eqref{mle-var} form the M-steps of the EM algorithm.  The full algorithm iterates between the E- and M-steps until convergence.

\textbf{Relaxing orthogonality.} 
\looseness-1
\Cref{ass:orthonormal-A} can be relaxed to ${(\bA^\star)^\top \bA^\star\!=\! \left(\bm\Lambda^\star \right)^2}$ for an unknown positive semi-definite (p.s.d.) matrix $\bm\Lambda^\star$. This is equivalent to assuming that $\bA^\star \!=\! \bO^\star \bm\Lambda^\star$ for some orthonormal matrix $\bO^\star$. Accordingly, we parametrize ${\bA \!=\! \bO \bm\Lambda}$ with $\bO \!\in\! \cO^{d_X \times d_Z}$ and $\bm\Lambda$ a p.s.d.\ matrix. The only modifications needed are
to replace $\bD^{-1}$ with $\bm\Lambda^{-1}$ in Tweedie's update in~\eqref{eq:Tweedie} and replace the steps in~\eqref{A-update} with updates for $\bO$ and~$\bm\Lambda$ that solve \Cref{mle-A}. The second-order Tweedie update from~\eqref{eq:second-order-Tweedie} also generalizes: with $\Jb_{\widehat s}(\by_{ei},\ba_e)$ denoting the Jacobian of $\widehat s(\cdot,\ba_e)$ evaluated at $\by_{ei}$, we obtain~\citep[Thm.~1]{meng2021estimating}: 
\begin{equation*}
 \widehat{\bz_{eij}^2}
=
\widehat{\bz}_{eij}^2
+
\left[
\widehat\sigma^2\widehat{\bm\Lambda}^{-2}
+
\widehat\sigma^4\widehat{\bm\Lambda}^{-1}\Jb_{\widehat s}(\by_{ei},\ba_e)\widehat{\bm\Lambda}^{-1}
\right]_{jj}.
\end{equation*}

\section{Causal Score Matching}
\label{sec:causal_score_matching}
Approximating the posterior of $\zb$  via Tweedie's formula in~\cref{eq:tweedie_reduction,eq:Tweedie,eq:second-order-Tweedie} involves the score function~$\nabla \log {f}_{\ba}(\by)$. We propose a scalable method to estimate this score using
knowledge of the causal directed acyclic graph (DAG) $\cG^\star$ from \Cref{ass:SCM-prior}.

\textbf{Score decomposition.}
\looseness-1 
The true marginal ${f}_{\ba_e}(\by_{ei})$ is determined by the normal means model in~\eqref{eq:normal-means} and
depends on
$\ba_e$ only through
$\zb_{ei}\sim\rmp^\star(\bz \mid \ba_e)$.
Under~\cref{ass:SCM-prior}, 
the score of $\bz$ follows from the Markov factorization in~\cref{eq:causal_Markov_factorisation} w.r.t.\ the true
causal graph~$\cG^\star$. Its $j$\textsuperscript{th} component is given by
\begin{equation} \label{Markov-score}
\begin{aligned}
\big[\nabla \log\rmp^\star&(\bz \mid \ba)\big]_j
=
\partial_j \log \rmp^\star\big(z_j\mid \bz_{\pa(j)}, a_j\big) \\
&+\textstyle\sum_{k\in\ch(j)} \partial_j \log \rmp^\star\big(z_k \mid \bz_{\pa(k)}, a_k\big)
\end{aligned}    
\end{equation}
and  only depends on the Markov blanket~\citep{pearl1988probabilistic} of $z_j$ (i.e., the parents and children of $j$, and all other parents of children of~$j$) and the intervention target labels for $z_j$ and its children.

Yet, even though $\yb$ is simply a noisy, element-wise rescaling of~$\zb$, the decomposition from~\cref{Markov-score} need not hold for the score $\nabla \log f_{\ba}(\by)$ due to the additional noise term~$\zetab$ in~\cref{eq:normal-means}. Instead, the factorization
of
$
f_{\ba}(\by)$ is determined by d-separation~\citep{Pearl2009} in the extended graph of $\cG^\star$ with additional nodes for $\yb$ and $\ab$ and additional edges $\{a_j\to z_j\to y_j\}_{j\in[d_Z]}$, see~\Cref{fig:true_chain_dag} for an example.

\looseness-1 In general, each component of $\nabla \log f_{\ba}(\by)$ may depend on most or all of $\yb$ and~$\ab$, i.e., it is a dense
vector field \citep[\S13.1]{Bishop2006PatternLearning}. This poses challenges for estimation, particularly when $d_Z$ is large. We therefore propose a sparse approximation of the true score that incorporates prior causal knowledge. 

\textbf{Causal score function.}
Specifically, we
consider the class of \emph{causal score functions}
$s:\R^{d_Z} \times \{0,1\}^{d_Z}\to\R^{d_Z}$ 
w.r.t.\ a given DAG~$\cG$ over~$[d_Z]$, defined component-wise by 
\begin{equation} 
\label{s-function}
\left[s(\by, \ba) \right]_j := s_j\!\left(y_j,\yb_{\pa(j;\Gcal)},a_j\right),
\end{equation}
where
$\pa(j;\Gcal)$ refers to the parents in~$\cG$.
Natural choices of $\Gcal$ include $\Gcal=\Gcal^\star$, the empty graph, and the full DAG induced by a known causal order. More generally, one can also use a complete undirected score model, in which each component $s_j$ is allowed to depend on $(y_j,\yb_{-j},a_j)$. These choices amount to %
different score approximations.

Using a sparser DAG can be beneficial from a bias--variance perspective: while the %
class~\eqref{s-function} may be biased for the full score, restricting inputs to $\pa(j)$ can reduce the variance of score estimation relative to conditioning on a larger set. It also reduces computational cost and thus scales better to large DAGs. Conversely, a denser graph or complete undirected score model can reduce approximation bias, but generally increases variance and computational cost.

\textbf{Score matching.} We fit the causal score function $s$ via score-matching~\citep{hyvarinen2005estimation} w.r.t.\ the true score, which on the population level is equivalent to minimizing the following $L_2$ loss,
\begin{equation}
\label{eq:L2_score_matching_loss}
    \hspace{-0.5em}
    {L}^{\cG}\left( s \right)
    \!=\!
    \sum_{e=1}^M
    \EE{\by \sim f_{\ba_e}}{\norm{s(\by,\ba_e)\!-\!\nabla \log f_{\ba_e}\!\left(\by\right)}^2_2}\!.
\end{equation}
If the score class is unrestricted, for example under the complete undirected score model with a sufficiently rich function class, the population minimizer of~\cref{eq:L2_score_matching_loss} coincides with the full score $\nabla \log f_{\ba}(\by)$. The same is true for the graph-restricted class when the true graph is empty. More generally, the causal score function is a sparse, structured projection of the full score onto the class determined by~$\cG$.
When $\cG = \cG^\star$ and $\sigma^2=0$ (and thus $\yb=\Db\zb$), the optimal causal score recovers the true score for leaf nodes~$j$ 
(for which ${\ch(j)=\varnothing}$ and the second term on the RHS of~\cref{Markov-score} vanishes).
Otherwise, it serves as a biased but lower-variance approximation.

\begin{figure}[tbp]
\centering
\begin{tikzpicture}[thick,
    scale=1, transform shape,
    node distance=1.4cm,
    every node/.style={draw=black, circle,
    font=\normalsize}
]
\node (z1) {$z_1$};
\node (z2) [right=1cm of z1] {$z_2$};
\node (z3) [right=1cm of z2] {$z_3$};

\node[below=0.35cm of z1] (y1) {$y_1$};
\node[below=0.35cm of z2] (y2) {$y_2$};
\node[below=0.35cm of z3] (y3) {$y_3$};

\node[obs, above=0.35cm of z1] (a1) {$a_1$};
\node[obs, above=0.35cm of z2] (a2) {$a_2$};
\node[obs, above=0.35cm of z3] (a3) {$a_3$};

\draw[->,-stealth] (z1) -- (y1);
\draw[->, -stealth] (z2) -- (y2);
\draw[->, -stealth] (z3) -- (y3);
\draw[->,-stealth] (a1) -- (z1);
\draw[->, -stealth] (a2) -- (z2);
\draw[->, -stealth] (a3) -- (z3);
\draw[->, -stealth] (z1) -- (z2);
\draw[->, -stealth] (z2) -- (z3);
\end{tikzpicture}
\caption{\looseness-1 \textbf{Example DAG with intervention targets and surrogate latents.} The chain graph over $\bz$ implies ${z_1\independent z_3 \mid z_2}$
and $z_j \independent a_{k} \mid z_{j-1}$ for $k\neq j$. Yet, due to measurement error, 
${y_1\not\independent y_3 \mid y_2}$
and $y_j \not\independent a_l \mid \yb_{-j}$ for $l\leq k$.
This illustrates why $\nabla \log f_{\ba}(\by)$ is generally dense.
}
\label{fig:true_chain_dag}
\end{figure}
\begin{figure*}[t]
\vspace{-0.25em}
\centering
\begin{minipage}[t]{0.32\textwidth}
\centering\vspace{0pt}
{\small (a)}\par
\includegraphics[width=\linewidth]{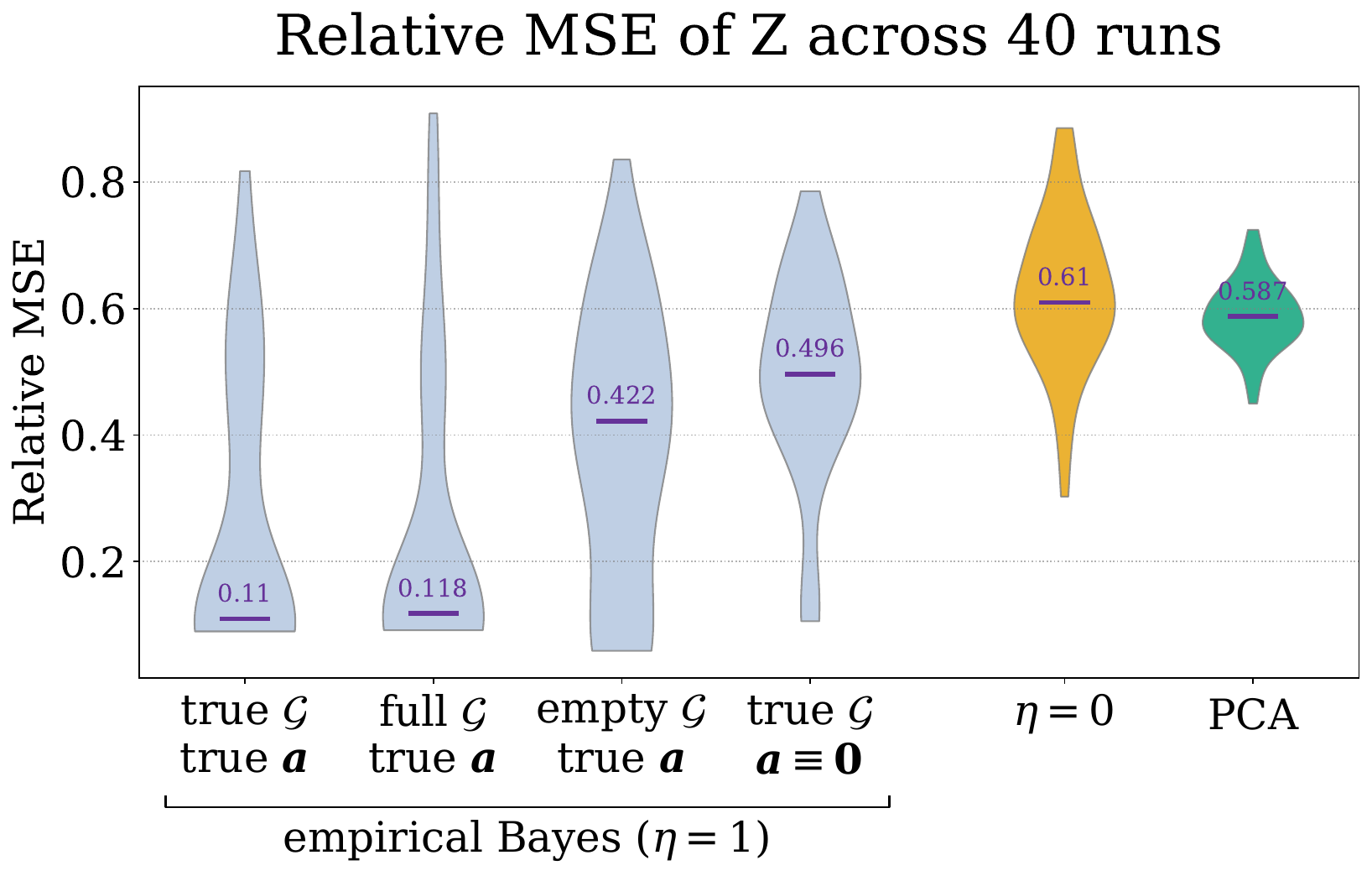}
\end{minipage}\hfill
\begin{minipage}[t]{0.32\textwidth}
\centering\vspace{0pt}
{\small (b)}\par
\includegraphics[width=\linewidth]{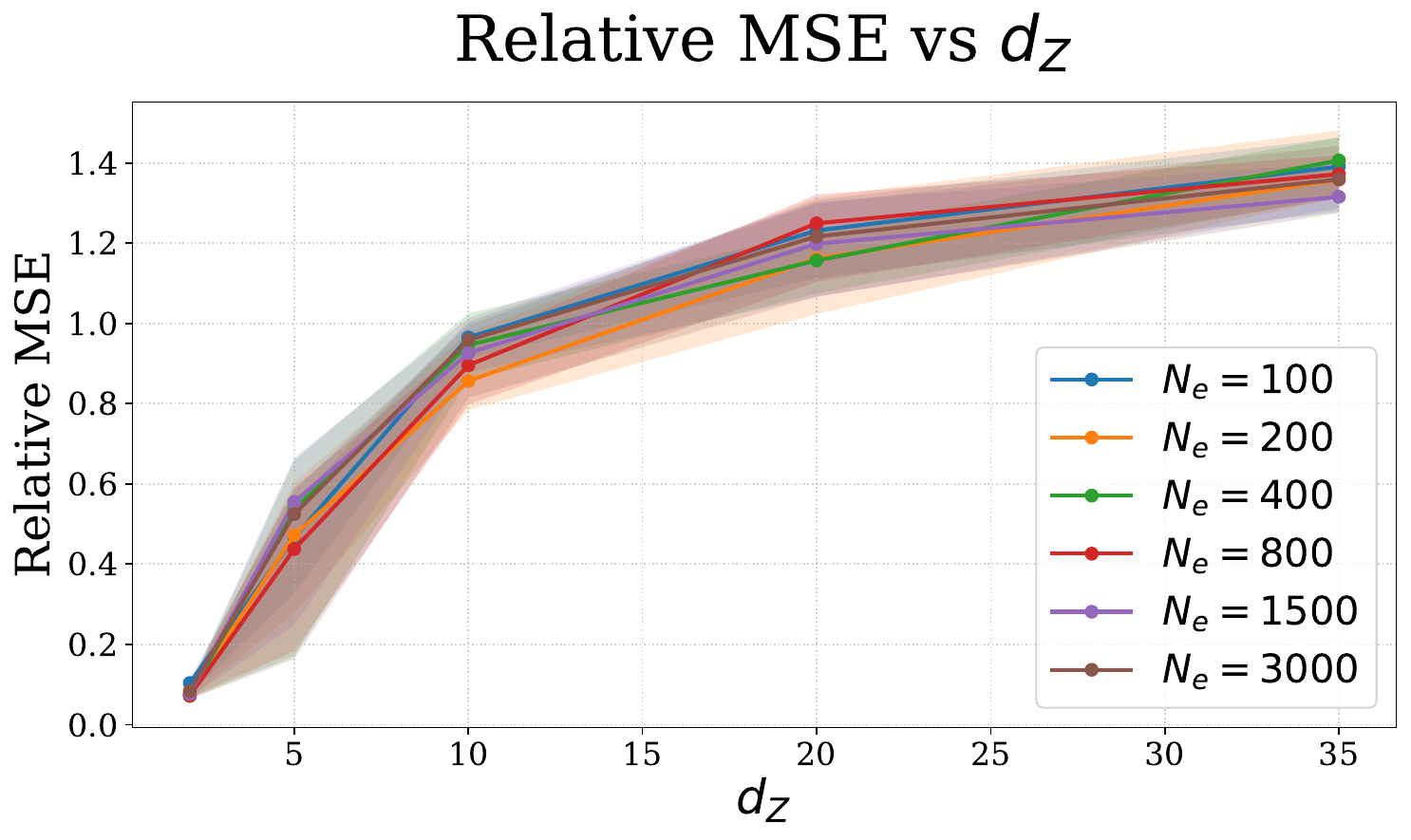}
\end{minipage}\hfill
\begin{minipage}[t]{0.32\textwidth}
\centering\vspace{0pt}
{\small (c)}\par
\includegraphics[width=\linewidth]{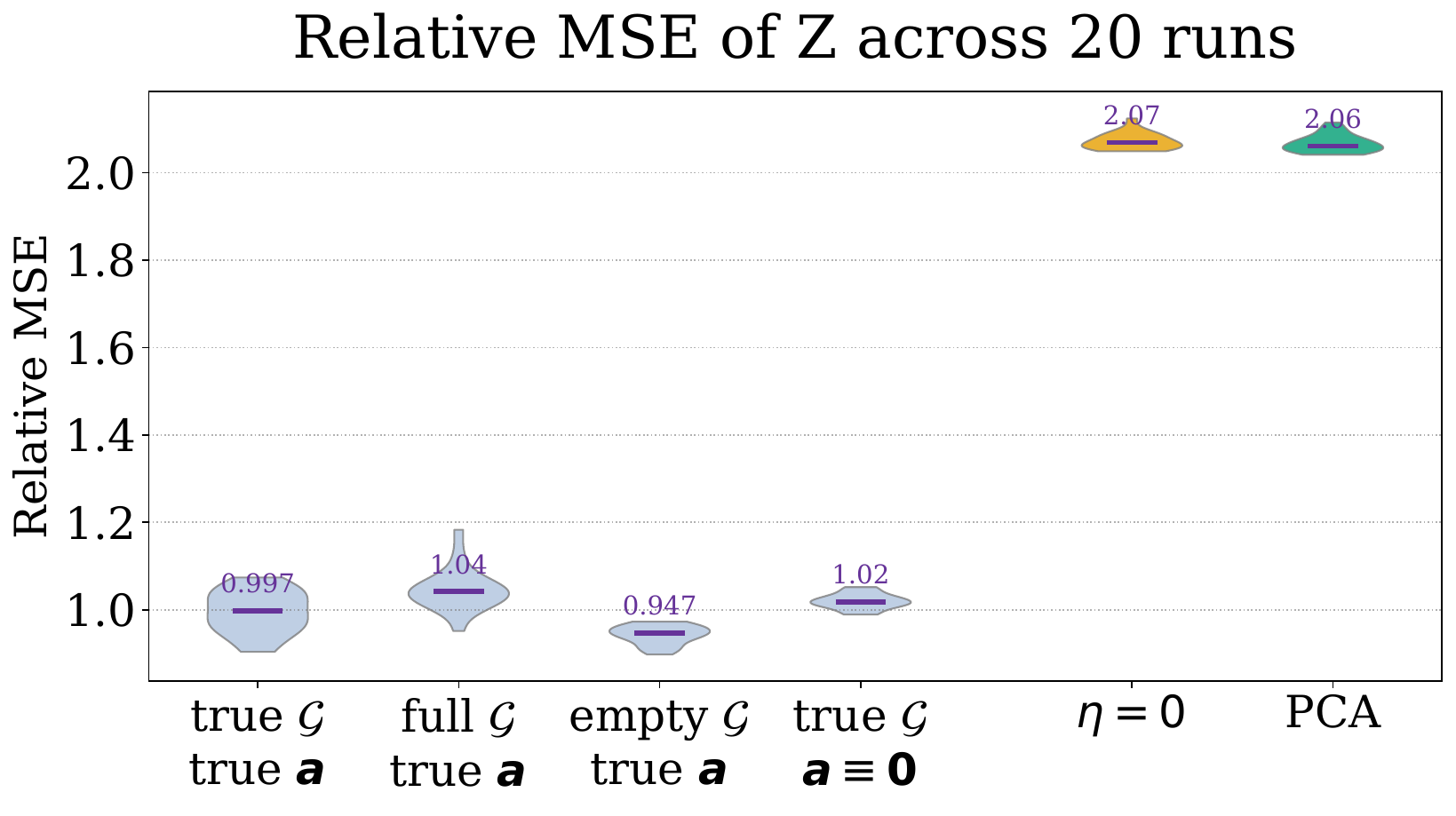}
\end{minipage}
\vspace{0.25em}

\begin{minipage}[t]{0.32\textwidth}
\centering\vspace{0pt}
{\small (d)}\par
\includegraphics[width=\linewidth]{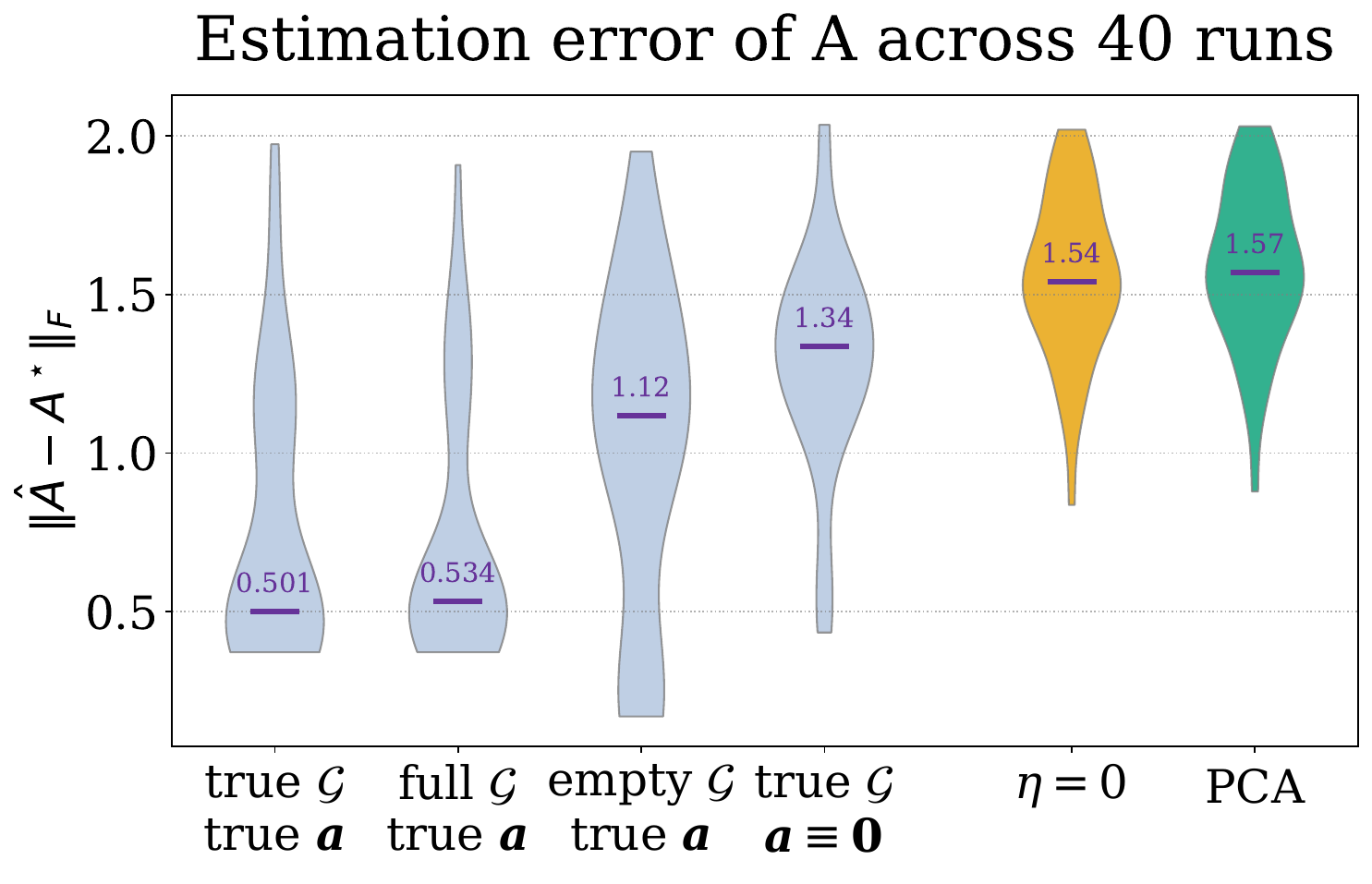}
\end{minipage}\hfill
\begin{minipage}[t]{0.32\textwidth}
\centering\vspace{0pt}
{\small (e)}\par
\includegraphics[width=\linewidth]{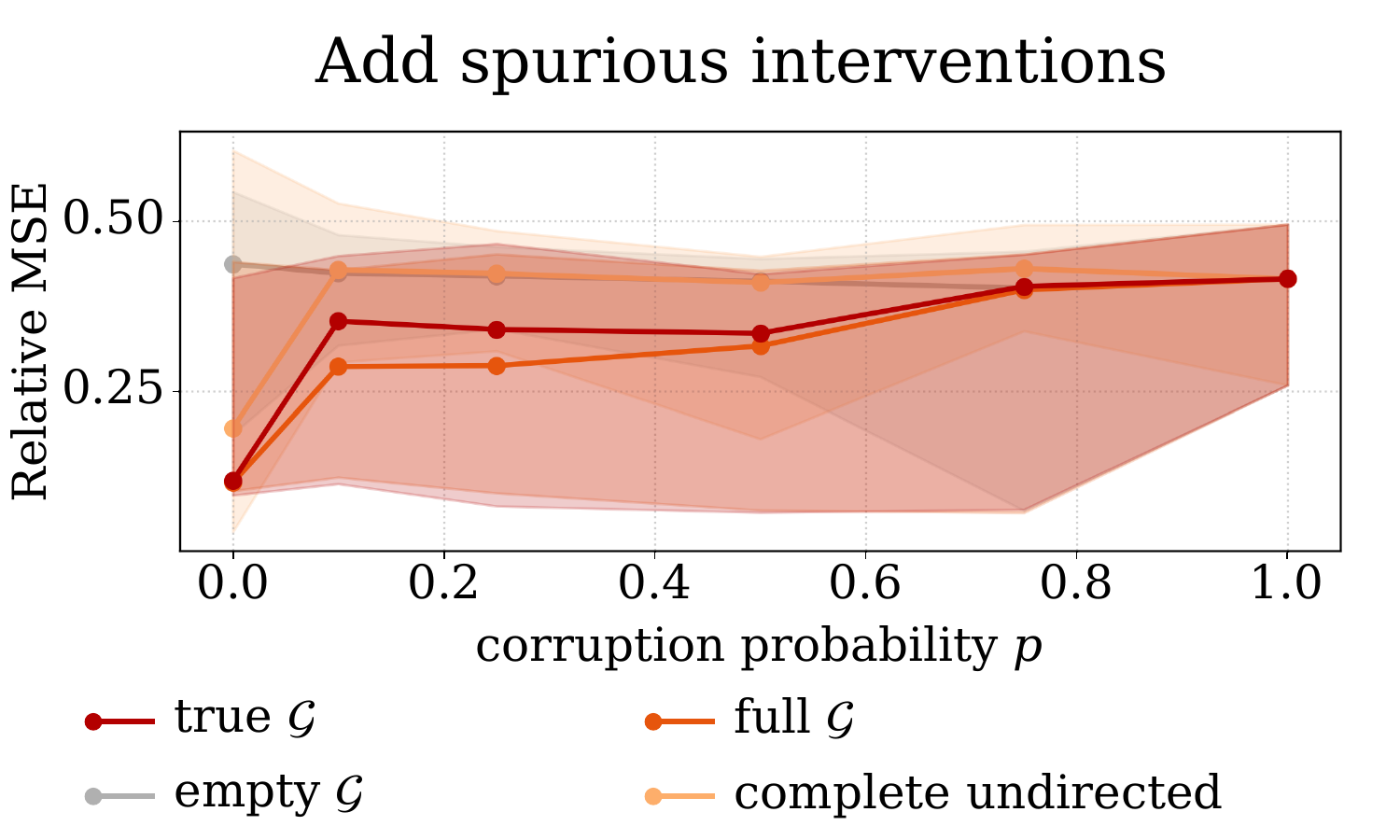}
\end{minipage}\hfill
\begin{minipage}[t]{0.32\textwidth}
\centering\vspace{0pt}
{\small (f)}\par
\includegraphics[width=\linewidth]{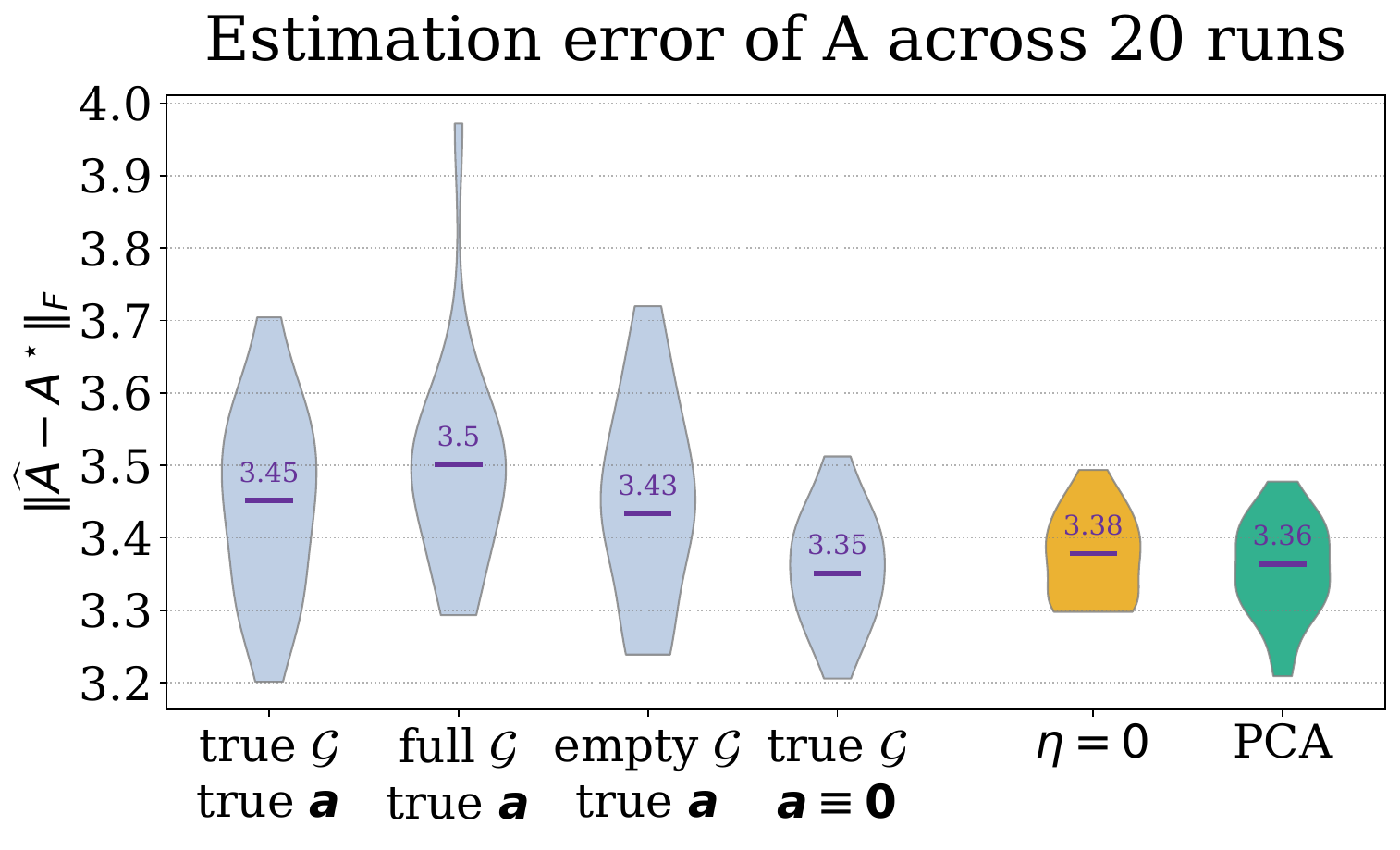}
\end{minipage}
\vspace{-0.25em}
\caption{\textbf{Empirical studies of CRL $f$-modeling.} Left and center columns show results on synthetic data: (a)~latent relative MSE and (d)~mixing-matrix Frobenius error across $40$ runs; (b)~relative MSE by latent dimension $d_Z$ and per-environment sample size; (e)~robustness to spuriously labelling each non-intervened node as intervened with probability $p$. Curves and bands in (b,e) show medians and interquartile ranges. The right column shows results on the semi-synthetic protein-signaling data: (c)~latent relative MSE and (f)~mixing matrix Frobenius error across $20$ runs.
In (c), the full-DAG violin omits one off-scale run, while its median uses all $20$. Lower values indicate better recovery. 
}
\vspace{-0.25em}
\label{fig:main-results}
\end{figure*}

Besides generally enforcing sparsity, 
the specific choice of arguments to $s_j$ on the RHS of~\cref{s-function} is informed by adherence to the causal generative process.  Under the assumed interventional SCM priors~(\Cref{ass:SCM-prior}), each latent $z_j$ depends only on $\zb_{\pa(j
)}$ and $a_j$.  However, this factorization generally need not hold for the marginal distribution of the Tweedie estimates $\{\widehat \zb_{ei}\}_{i\in[N_e]}$ from \eqref{eq:Tweedie} as $N_e\to\infty$, even in the noiseless case
, since \eqref{Markov-score} also includes contributions from children, their intervention targets, and other parents. 

If we view $y_j$ as a noisy surrogate of~$z_j$, it therefore seems natural to posit that the estimator $\widehat z_j$ derived from~\eqref{eq:tweedie_reduction} should depend only on $(y_j,\yb_{\pa(j; \cG)},a_j)$.
If $\cG=\cG^\star$, this restriction guarantees that, as $N_e\to\infty$, the  empirical distribution $\frac{1}{N_e} \sum_{i = 1}^{N_e} \delta_{\widehat z_{eij}}$ of the Tweedie estimates $\{\widehat z_{eij}\}_{i\in[N_e]}$ in \Cref{eq:Tweedie} depends only on
~$(z_j,\zb_{\pa(j)},a_j)$. 
Moreover, the approximation 
$[s(\cdot,\ba)]_j =s_j(\cdot,a_j)$ becomes accurate in the large-sample regime, as discussed in~\Cref{app-sparse-intervention}.

\textbf{Decoupled estimation.}
The following result characterizes the
structure for the minimizer of the score-matching loss.
\begin{restatable}
{theorem}{solutionstructure}
\label{thm:score-matching-minimizer}
 For all solutions $s^\star \in\argmin_{s\in\cS}{L}^{\cG}(s)$ 
 and all
 $j\in[d_Z]$, the $j$\textsuperscript{th} component $s^\star_j$ of $s^\star$ minimizes
 \vspace{-0.5em}
\begin{equation}
\label{eq:score-matching-j}
    \sum_{e=1}^M  \mathbb{E}_{\by\sim {f}_{\ba_e}}
    \!\Big[
    \big| s_j\big(y_j,\yb_{\pa(j)}, a_{ej}\big)\big|^2
    \!+2\partial_j s_j\big(y_j, \yb_{\pa(j)}, a_{ej}\big)
    \Big]\!.
\end{equation}
\end{restatable}
\vspace{-.75em}
\Cref{thm:score-matching-minimizer} provides a key structural insight into the causal score-matching problem: The optimization over $s$ decouples across the components $\{s_j\}_{j=1}^{d_Z}$. To learn the causal score, it thus suffices to solve $d_Z$ smaller score-matching problems. 

\textbf{Empirical considerations.}
In practice, we only have the finite sample from~\eqref{eq:data} and do not know $\Ab$ (nor $\yb)$.
At each iteration of~\cref{alg:iterative-SM}, we therefore use the current estimate $\widehat{\Ab}=\widehat\Ob\widehat\Db$ of~$\bA$ to form the pseudo-latents $\by_{ei}=\widehat\bO^\top \bx_{ei}$ and estimate each $\widehat s_j$ by minimizing the empirical version~$\widehat{L}_j^{\cG}$ of the loss in~\cref{eq:score-matching-j} obtained by
replacing the expectation
by an empirical average over $\{\yb_{ei}\}_{i=1}^{N_e}$:
\begin{equation} \label{score-matching}
\textstyle
\widehat s_j \in \arg\min_{s_j} \widehat{L}^{\cG}_j(s_j).
\end{equation}
Once we obtain the estimates $\{\widehat s_j\}_{j\in[d_Z]}$, the full score estimate $\widehat s$ is a $d_Z$-dimensional vector field such that
\begin{equation*} 
     [\widehat s(\by, \ba)]_j = \widehat s_j(y_j,\yb_{\pa(j)},a_j).
\end{equation*}

\textbf{Estimation with cubic splines.}
In practice, each score component $s_j$ is parameterized by a nonparametric function class \citep{Tsybakov2009}. An example is the family of cubic splines, which gives a simple closed-form solution (see \Cref{app-score-cubic-splines}). The partial derivatives are computed efficiently, either in closed form or via automatic differentiation.

\section{Empirical Studies}
\label{sec:experiments}

\textbf{Methods and Baselines.} To investigate the role of the graph~$\Gcal$ used to approximate the true score~(\cref{sec:causal_score_matching}), we evaluate  CRL $f$-modeling~(\cref{alg:iterative-SM}) with $\eta=1$ and true intervention labels $\ba$ for different choices of DAGs $\Gcal\in$ \{true, full DAG, empty\} and a complete undirected graph.
To test the value of known intervention targets, we also consider a pooling variant which
retains the true graph, but ignores the multi-domain structure by setting $\ba_e\equiv\bm 0$ for all~$e$. 
As methods without EB shrinkage, we consider a version of~\cref{alg:iterative-SM} with 
$\eta=0$, which does not use score estimation and hence requires neither a graph nor intervention targets; and a baseline that performs PCA on the pooled data. 

\textbf{Metrics.}
\looseness-1
We measure latent recovery by \emph{relative MSE}%
:
\begin{equation}\label{eq:rel-mse}
\mathrm{RelMSE}(\widehat\bz,\bz^\star)
:=\frac{\|\widehat\bz-\bz^\star\|_2^2}{\|\bz^\star\|_2^2}.
\end{equation}
Mixing-matrix recovery is measured by Frobenius error:
\begin{equation}\label{eq:Aerr}
\norm{\widehat\bA - \bA^\star}_F
= \left[\textstyle\sum_{i=1}^{d_X}\sum_{j=1}^{d_Z}\left(\widehat\bA_{ij}-\bA^\star_{ij}\right)^2\right]^{1/2}.
\end{equation}
We align the final estimate $\widehat\bA$ from \Cref{alg:iterative-SM} to $\bA^\star$ by a signed column permutation, which is consistent with identifiability up to this ambiguity (\Cref{sect-identifiability}).

\subsection{Simulation With Synthetic Data}
\textbf{Data.} We simulate a nonlinear SCM with ${d_Z=4}$ latents, chain causal graph ($z_1\!\to\! z_2\!\to\! z_3 \!\to\! z_4$), and domains generated via single-node hard interventions. A column-orthonormal mixing matrix maps the latents to $d_X=100$-dimensional observations with additive Gaussian noise. We generate $40$ independent draws of the measurement model and use a sample size of $N_e=2{,}000$ per domain. \Cref{simulation-details,app-alg-details} provide  further simulation and %
algorithmic 
details.

\textbf{Results.}
\Cref{fig:main-results}~(a,d) reports both metrics across $40$ runs:
\begin{enumerate}[leftmargin=*,topsep=0em,itemsep=0em]
    \item The true and full DAG variants outperform other methods. The full DAG provides a richer score approximation that respects causal ordering (\Cref{sec:causal_score_matching}).
    The empty graph reduces statistical complexity but is less stable across runs, underperforming when causal-parent effects are strong.
    \item Pooling ($\ba\equiv\bm 0$)
    trails other EB methods, highlighting the value of modeling differences across domains.
    \item All four empirical Bayes variants outperform the no-shrinkage baselines ($\eta=0$, PCA).
\end{enumerate}
In \Cref{fig:main-results}~(b), the relative MSE of true-DAG method grows roughly as $O(\sqrt{d_Z})$ and stabilizes once $N_e \ge 200$. Computational cost scales as $O\big(\exp\big(\max_{j \in [d_Z]}|\pa(j)|\big)\big)$ for the true DAG and $O(\exp(d_Z))$ for the full DAG, limiting the scalability of the latter method at large $d_Z$ (\Cref{app-score-cubic-splines}).

\Cref{fig:main-results}~(e) assesses robustness to misspecification of~$\ab$ by spuriously labeling each non-intervened node as intervened with probability $p$ in each environment. Performance degrades gracefully as $p$ increases, suggesting that our method is robust to moderate target over-specification.

Additional simulations varying the DAG, measurement and intervention noise, latent dimension, evaluation metric, and damping factor $\eta$ are provided in~\Cref{simulation-details}. Across settings, the true-DAG and causally ordered full-DAG fits often achieve the most accurate latent recovery. For instance, \Cref{fig:per-env} in \Cref{simulation-details} compares per-environment relative MSE with oracle ($\bA^\star,\sigma^{2\star}$) and jointly estimated measurement models. When $\bA$ is learned, the true-DAG fit has the lowest overall error, relatively uniformly across environments. 

Beyond the chain DAG, we also apply our method to a  semi-synthetic single-cell study with a real myeloid regulatory network~\citep{krumsiek2011hierarchical} as causal graph and simulated latents. In that example, CRL $f$-modeling with the  true DAG and the empty DAG achieve the lowest errors among the six methods (\cref{fig:app-learned}), see~\Cref{sec:application} for further details on the data-generating process, implementation, and results.

\subsection{Semi-Synthetic Single-Cell Protein-Signaling Data}
\label{sec:sachs}
\textbf{Data.}
To study our approach in a more realistic setting without fully known generative process, we create semi-synthetic data based on the single-cell protein-signaling study of~\citet{sachs2005causal}. The study concerns $11$ phosphorylated proteins measured in thousands of human immune cells across $9$ conditions, of which $7$ stimulatory or inhibitory with know protein targets. We treat these conditions as domains, the measured protein activities as ground-truth latents $\bz\in\RR^{11}$, 
and the curated consensus network (\Cref{fig:sachs-dag}, an acyclic proxy for a system with feedback) as the causal graph. We then simulate measurements $\bx=\bA\bz+\epsilonb$ with column-orthonormal $\bA\in\cO^{d_X\times d_Z}$ and $d_X=100$. This setup probes recovery with real latents and interventions under correctly specified linear-Gaussian measurements. 

\textbf{Results.} The results across 20 independent runs with different unknown measurement models are shown in~\cref{fig:main-results}~(c,f).
\begin{enumerate}[leftmargin=*,topsep=0em,itemsep=0em]
    \item All four EB variants yield substantially better estimates of $\zb$ (roughly half the median RelMSE) compared to~\cref{alg:iterative-SM} without shrinkage ($\eta=0$) and PCA.
    \item Differences across methods in estimation of $\Ab$ are less pronounced. The pooling approach and  no-shrinkage baselines attain the lowest median Frobenius error.
\end{enumerate}
The $7$ targeted conditions cover only $6$ of the $11$ proteins, with two targeting the same protein in different ways (see~\Cref{tab:sachs-interventions} in \Cref{app:sachs} for details). This limited coverage, the proxy nature of the causal DAG, and conflicting interventions 
may limit the benefit of target labels, which could explain the competitiveness of the pooling approach ($\ab\equiv \bm0$). Moreover, the empty graph's lower-dimensional score model may reduce variance and sensitivity to graph misspecification. 
Score approximation under the proxy DAG may alter mixing directions and correlation-based alignment, potentially explaining the slightly higher mixing error than PCA despite better latent recovery.

In~\Cref{app:sachs}, we also report mean correlation coefficients for the recovered latents (\Cref{fig:sachs-mcc}) and estimation accuracy with known $\bA$ (\Cref{fig:sachs-known}). The trends remain similar.
\looseness-1 
\section{Discussion}
\label{sec:discussion}
We proposed an empirical Bayes $f$-modeling approach to multi-domain causal representation learning that addresses the simultaneous inference of latent causal variables under a noisy measurement model. The main EM algorithm performs iterative Tweedie updates with score matching. The empirical results show accurate performance. 

Our work focuses on estimation under a given causal structure. We assume the causal graph and intervention targets are available and use them to construct a structured score approximation. Thus, our method should not be read as a causal structure learning procedure; rather, it is a causality-oriented dimensionality reduction method for accurately estimating $\bz$ and $\bA$ once a causal factorization is specified.

When the causal DAG is unknown, one option is to run CRL $f$-modeling with several candidate graphs and select among them using suitable goodness-of-fit criteria. Alternatively, one can use the empty graph, which favors low variance but may incur high bias, or the complete undirected score model, which has lower approximation bias but higher variance. If a causal order is known, the full DAG induced by that order provides a practical compromise. When the DAG is large, CRL $f$-modeling with the true DAG, or with a sparse approximation thereof, is typically the more scalable choice.

 For nonlinear measurement models, $f$-modeling is less directly applicable because of the difficulty to adapt Tweedie's formula. On the other hand, $\rmg$-modeling becomes more natural since it extends directly to nonlinear or learned decoders $\rmp_{\thetab}(\bx\mid\bz)$, including deep networks, see~\cref{sec:EB_g_modeling} for details.

Fitting the scores with smoothing splines scales as $O(\exp(d_Z))$ due to a $d_Z$-dimensional tensor product. One could use approximate estimation methods that scale polynomial in the dimension of the latent space.  

Causal empirical Bayes may extend to other algorithms and model classes, beyond the scenario we studied. Extensions include richer intervention structures $\rmp(\bz\mid\ba)$ with non-binary $\ba$ (e.g., additive shifts) and learning unknown graphs or targets using the optimal EB objectives.

\acknowledgments{The authors thank Jonas Peters for insightful discussions during the conception of this work. JvK acknowledges support from The Branco Weiss Fellowship---Society in Science.}

\FloatBarrier
\subsection*{AI Use Statement}
The authors completed the main manuscript in February 2026 with minimal assistance from generative AI. We subsequently used generative AI tools to help develop the real-data extension, including implementation and debugging, and preparation of figures. The authors take full responsibility for the content of this paper.

\bibliography{bib}

\clearpage
\section*{Checklist}
\begin{enumerate}

  \item For all models and algorithms presented, check if you include:
  \begin{enumerate}
    \item A clear description of the mathematical setting, assumptions, algorithm, and/or model. [Yes] 
    \item An analysis of the properties and complexity (time, space, sample size) of any algorithm. [Yes]
    \item (Optional) Anonymized source code, with specification of all dependencies, including external libraries. [Yes] 
  \end{enumerate}

  \item For any theoretical claim, check if you include:
  \begin{enumerate}
    \item Statements of the full set of assumptions of all theoretical results. [Yes]
    \item Complete proofs of all theoretical results. [Yes]
    \item Clear explanations of any assumptions. [Yes] 
  \end{enumerate}

  \item For all figures and tables that present empirical results, check if you include:
  \begin{enumerate}
    \item The code, data, and instructions needed to reproduce the main experimental results (either in the supplemental material or as a URL). [Yes]
    \item All the training details (e.g., data splits, hyperparameters, how they were chosen). [Yes]
    \item A clear definition of the specific measure or statistics and error bars (e.g., with respect to the random seed after running experiments multiple times). [Yes] 
    \item A description of the computing infrastructure used. (e.g., type of GPUs, internal cluster, or cloud provider). [Yes]
  \end{enumerate}

  \item If you are using existing assets (e.g., code, data, models) or curating/releasing new assets, check if you include:
  \begin{enumerate}
    \item Citations of the creator If your work uses existing assets. [Yes] 
    \item The license information of the assets, if applicable. [Yes]
    \item New assets either in the supplemental material or as a URL, if applicable. [Yes] 
    \item Information about consent from data providers/curators. [Not Applicable]
    \item Discussion of sensible content if applicable, e.g., personally identifiable information or offensive content. [Not Applicable] 
  \end{enumerate}

  \item If you used crowdsourcing or conducted research with human subjects, check if you include:
  \begin{enumerate}
    \item The full text of instructions given to participants and screenshots. [Not Applicable] 
    \item Descriptions of potential participant risks, with links to Institutional Review Board (IRB) approvals if applicable. [Not Applicable]
    \item The estimated hourly wage paid to participants and the total amount spent on participant compensation. [Not Applicable] 
  \end{enumerate}

\end{enumerate}

\clearpage

\onecolumn

\section*{Supplementary Material}
\appendix

\section{Details of the \texorpdfstring{$f$}{f}-modeling Algorithm} \label{app-alg-details}
In this appendix, we fill in the detailed derivation of \Cref{alg:iterative-SM}.   The full procedure is a variational EM algorithm \citep{wainwright2008graphical}.  

The marginal likelihood of $(\bA,\sigma^2)$ is
\begin{equation*}
\sum_{e = 1}^M \sum_{i = 1}^{N_e} \log \rmp_{\bA,\sigma^2}^\star(\bx_{ei} \mid \ba_e)
= \sum_{e = 1}^M \sum_{i = 1}^{N_e} \log \int_{\R^{d_Z}} \rmp_{\bA,\sigma^2}(\bx_{ei} \mid \bz)\,\rmp^\star(\bz \mid \ba_e)\,\dd \bz .
\end{equation*}

The log-likelihood is difficult to optimize because it involves intractable integrals, so we approximate it using variational inference. In the E-step, we lower bound the log-likelihood by an ELBO. For any $\rmq_{ei}\in \cP(\R^{d_Z})$,
\begin{equation*}
\begin{aligned}
\log \rmp_{\bA,\sigma^2}^\star(\bx_{ei} \mid \ba_e)
&\geq \EE{\rmq_{ei}(\bz_{ei})}{\log \rmp_{\bA,\sigma^2}(\bx_{ei} \mid \bz_{ei})}
- \KL{\rmq_{ei}(\bz_{ei})}{\rmp^\star(\bz \mid \ba_e)} \\
&= -\frac{1}{2\sigma^2}\,\EE{\rmq_{ei}(\bz_{ei})}{\|\bx_{ei}-\bA\bz_{ei}\|^2}
-\frac{d_X}{2}\log\sigma^2
- \KL{\rmq_{ei}(\bz_{ei})}{\rmp^\star(\bz \mid \ba_e)} \\
&= -\frac{1}{2\sigma^2}\,\EE{\rmq_{ei}(\bz_{ei})}{\|\bx_{ei}\|^2 + \bz_{ei}^\top \bD^2 \bz_{ei} - 2\,\bx_{ei}^\top \bA\bz_{ei}}
-\frac{d_X}{2}\log\sigma^2
- \KL{\rmq_{ei}(\bz_{ei})}{\rmp^\star(\bz \mid \ba_e)} \\
&=: \ELBO^\star(\rmq_{ei},\bA,\sigma^2).
\end{aligned}
\end{equation*}
Then we design an EM algorithm which aims to iteratively maximize the objective $\sum_{e = 1}^M \sum_{i = 1}^{N_e} \ELBO^\star(\rmq_{ei}, \bA, \sigma^2)$. 

\textbf{Update $\bz$: } For given $\bA$ and $\sigma^2$, the distribution $\rmq_{ei}$ that maximizes the ELBO is the posterior
\begin{equation*}
    \rmq_{ei}^\star(\bz_{ei}) \propto \rmp_{\bA,\sigma^2}(\bx_{ei} \mid \bz_{ei})\,\rmp^\star(\bz_{ei} \mid \ba_e).
\end{equation*}
A direct calculation shows that, for updating $\bA$ and $\sigma^2$ in the subsequent M-step, it suffices to keep track of the first and second moments of $\rmq_{ei}^\star$. To approximate these moments, we use the first- and second-order Tweedie formulas.

The score-matching problem can be expressed in terms of the diagonal score components $\left(\widehat s_j\right)_{j \in [d_Z]}$:
\begin{equation} \label{score-matching-jk}
\widehat s_j \in \argmin_{s_j} \sum_{e = 1}^M \sum_{i = 1}^{N_e}\Bigg[ \partial_j s_j\left(y_{eij}, y_{ei\pa(j)}, \ba_e \right)+ \frac{1}{2} s_j\left(y_{eij}, y_{ei\pa(j)}, \ba_e \right)^2 \Bigg]. 
\end{equation}
The resulting score estimate is given by $[\widehat s(\by,\ba)]_j=\widehat s_j(\by_j,\by_{\pa(j)},\ba)$, or, under the sparse-intervention approximation, $[\widehat s(\by,\ba)]_j=\widehat s_j(\by_j,\by_{\pa(j)},\ba_j)$, for $j\in[d_Z]$.

Given an estimate $\widehat\bD$ of the diagonal matrix $\bD$, we update the empirical Bayes estimate of the latent variable $\bz_{ei}$ using a damped version of Tweedie's formula:
\begin{equation*}
    \widehat{\bz_{ei}} = \widehat \bD^{-1} \left(\by_{ei} + \eta \sigma^2 \widehat{s}(\by_{ei}, \ba_e) \right), 
\end{equation*}
where $\eta\in(0,1)$ is a damping factor.

The second-order Tweedie's formula yields
\begin{equation*}
    \widehat{z_{eij}^2} = \widehat{z_{eij}}^2+ \frac{\widehat \sigma^2}{\widehat \bD_{jj}^2} +  \frac{ \widehat \sigma^4}{\widehat \bD_{jj}^2} \partial_j \widehat s_j(y_j, \yb_{\pa(j)}, \ba_e), \quad j \in [d_Z]. 
\end{equation*}

\textbf{Update $\bA$: } 
In the M-step, we maximize the ELBO over $\bA$, equivalently
\begin{equation}
        \widehat \bA  = \arg \min_{\bA \text{ orthogonal}} \sum_{e = 1}^M  \sum_{i = 1}^{N_e} \EE{\rmq(\bz_{ei})}{\bz_{ei}^\top \bD^2 \bz_{ei}- 2 \bx_{ei}^\top \bA \bz_{ei}}.
\end{equation}
This optimization depends on the posterior mean and second moments under $\rmq$. A natural choice is to take $\EE{\rmq}{\bz_{ei}}$ as the Tweedie estimate $\widehat{\bz_{ei}}$, and $\EE{\rmq}{\bz_{ei}^2} \in \R^{d_Z}$ as the second-order Tweedie estimate $\widehat{\bz_{ei}^2}$.

Define $\bM :=  \sum_{e = 1}^M \sum_{i=1}^{N_e} \bx_{ei} \widehat{\bz_{ei}}^\top$. The optimization problem reduces to
\begin{equation} \label{mle-A-tr}
\begin{aligned}
\widehat \bA
=\arg \min_{\bA} & \sum_{e = 1}^M \sum_{i=1}^{N_e} \left[\sum_{j = 1}^{d_Z} \bD_{jj}^2 \widehat{\bz_{eij}^2} - 2 \left\langle \bx_{ei}, \bA \widehat{\bz_{ei}} \right\rangle \right] \\
& = \sum_{e = 1}^M \sum_{i=1}^{N_e} \sum_{j = 1}^{d_Z} \bD_{jj}^2 \widehat{\bz_{eij}^2}  -2 \text{tr}\bigl(\bA \bM^\top \bigr).     
\end{aligned}
\end{equation}
where the second equality is due to 
\begin{equation*}
     \sum_{e = 1}^M \sum_{i=1}^{N_e} \left\langle \bx_{ei}, \bA \widehat{\bz_{ei}} \right\rangle = \sum_{e = 1}^M \sum_{i=1}^{N_e} \text{tr}\left(\bA \widehat{\bz_{ei}}  \bx_{ei}^\top \right) = \text{tr}\bigl(\bA\bM^\top \bigr). 
\end{equation*}
Now we parametrize $\bA = \bO \bD$ thanks to \Cref{ass:orthonormal-A}. 

Let $\bM = \bU \bm\Sigma \bV^\top$ be the singular value decomposition of $\bM$. Since the trace is cyclically invariant, the objective~\eqref{mle-A-tr} is equivalent to $\text{tr}\bigl(\bO \bD \bM^\top \bigr) =  \text{tr}\bigl(\bO \bV \bD \bm\Sigma^\top \bU^\top \bigr)$. By the orthogonal Procrustes solution \citep[§6.4.1]{golub2013matrix}, for a given $\bD$, the solution $\widehat \bO$ that maximizes \Cref{mle-A-tr} is given by
\begin{equation*}
    \widehat \bO = \bU \bV^\top \in \cO^{d_X \times d_Z}. 
\end{equation*}
After plugging $\widehat \bO$ into the reformulated objective, the optimal $\widehat \bD$ requires solving
\begin{align*}
\widehat \bD
& =\arg \min_{\bD}
\sum_{e = 1}^M \sum_{i=1}^{N_e} \sum_{j = 1}^{d_Z} \bD_{jj}^2 \widehat{\bz_{eij}^2} - 2\text{tr}\bigl(\bU \bD \bm\Sigma^\top \bU^\top \bigr)  \\
& = \sum_{e = 1}^M \sum_{i=1}^{N_e} \sum_{j = 1}^{d_Z} \bD_{jj}^2 \widehat{\bz_{eij}^2} - 2 \text{tr}\bigl(\bD \bm\Sigma^\top\bigr). 
\end{align*}
The problem reduces to solving a quadratic function for each entry of $\bD$. The optimizer is explicitly given by $\widehat \bD_{jj} = \bm\Sigma_{jj} / \sum_{e = 1}^M \sum_{i=1}^{N_e} \widehat{ \bz_{eij}^2}$ for $j \in [d_Z]$.

\textbf{Update $\sigma^2$. } Finally, we freeze $\{\widehat{\bz_{ei}}\}$, $\{\widehat{\bz_{ei}^2}\}$, and $\widehat \bA$, and maximize the ELBO with respect to $\sigma^2$. The ELBO is concave in $1/\sigma^2$, so setting the derivative to zero yields the unique maximizer
\begin{equation*}
    \widehat \sigma^2
    = \frac{1}{d_X \sum_{e=1}^M N_e}\sum_{e = 1}^M \sum_{i=1}^{N_e}
    \left(
    \|\bx_{ei}\|_2^2
    + \sum_{j = 1}^{d_Z} \widehat \bD_{jj}^2\,\widehat{\bz_{eij}^2}
    - 2\,\bx_{ei}^\top \widehat \bA \widehat{\bz_{ei}}
    \right).
\end{equation*}
\subsection{Score modeling with cubic splines} \label{app-score-cubic-splines}
To estimate the diagonal score components $\{s_j\}_{j\in[d_Z]}$, we use a simple spline-based nonparametric model \citep{Tsybakov2009} that admits closed-form updates. Fix $j\in[d_Z]$ and let $x_{eij}:=(y_{eij},y_{ei\pa(j)})\in\R^{|\pa(j)|+1}$ denote the input. We model
\begin{equation*}
s_j(y_{eij},y_{ei\pa(j)},a_{ej})
=\phi_j(x_{eij})^\top \theta_{j,a_{ej}},
\end{equation*}
where $\phi_j:\R^{|\pa(j)|+1}\to\R^m$ is a tensor-product cubic B-spline feature map with a fixed knot sequence on each coordinate, and $\theta_{j,a}\in\R^m$ are intervention-specific coefficients for $a\in\{0,1\}$. In particular, if we use $K$ basis functions per coordinate, then the tensor-product construction yields $m=K^{|\pa(j)|+1}$ features, so both memory and computational complexity grow exponentially in the input dimension $|\pa(j)|+1$. 

Write $\Psi_j(x):=\partial_{y_j}\phi_j(x)\in\R^m$ for the derivative of the spline features with respect to the first coordinate $y_j$. Then
\begin{equation*}
\partial_{y_j}\,s_j(y_{eij},y_{ei\pa(j)},a_{ej})
=\Psi_j(x_{eij})^\top \theta_{j,a_{ej}}.
\end{equation*}
Plugging this parameterization into the empirical score-matching loss for the $j$th component yields the quadratic objective
\begin{equation*}
\widehat{\bL}_j(\theta_{j0},\theta_{j1})
=\frac{1}{M}\sum_{e=1}^M\frac{1}{N_e}\sum_{i=1}^{N_e}
\left\{
\Psi_j(x_{eij})^\top \theta_{j,a_{ej}}
+\frac{1}{2}\bigl(\phi_j(x_{eij})^\top \theta_{j,a_{ej}}\bigr)^2
\right\}.
\end{equation*}
We fit $\theta_{j0},\theta_{j1}$ with ridge regularization,
\begin{equation*}
(\widehat\theta_{j0},\widehat\theta_{j1})
\in\arg\min_{\theta_{j0},\theta_{j1}}
\widehat{\bL}_j(\theta_{j0},\theta_{j1})
+\frac{\lambda}{2}\bigl(\|\theta_{j0}\|_2^2+\|\theta_{j1}\|_2^2\bigr),
\end{equation*}
which decouples over $a\in\{0,1\}$. For each $a$, define the index set $\cI_{ja}:=\{(e,i):a_{ej}=a\}$, the feature and derivative matrices
\begin{equation*}
\Phi_{ja}\in\R^{|\cI_{ja}|\times m},\quad \Phi_{ja}(r,:)=\phi_j(x_{eij})^\top,
\qquad
\Psi_{ja}\in\R^{|\cI_{ja}|\times m},\quad \Psi_{ja}(r,:)=\Psi_j(x_{eij})^\top,
\end{equation*}
and the normal equations
\begin{equation*}
\widehat\theta_{ja}
=-
\left(\Phi_{ja}^\top\Phi_{ja}+\lambda I_m\right)^{-1}\Psi_{ja}^\top \mathbf{1}.
\end{equation*}
The resulting score estimator is $\widehat s_j(x,a)=\phi_j(x)^\top \widehat\theta_{ja}$, and $\partial_{y_j}\widehat s_j(x,a)=\Psi_j(x)^\top \widehat\theta_{ja}$, which are the quantities used in the first- and second-order Tweedie updates.

\subsection{Addressing scale-permutation indeterminacy} 
From the objective~\eqref{mle-A}, the columns of $\bA$ are identifiable only up to signed permutations. Specifically, if $\bD_{\text{sign}}$ be a diagonal matrix with entries in $\{\pm 1\}$ and $\bP$ a permutation matrix, then objective~\eqref{mle-A} remains unchanged if we replace $\bD$ with $\bD \bD_{\text{sign}}$ and $\bO$ with $\bP \bO$. To obtain more stable updates, we select a canonical representative by replacing step~\eqref{A-update} with $\widehat \bA = \bU \bP \bV^\top \bD_{\text{sign}} \bD$, where $\bP$ is the permutation matrix that maximizes $\sum_{j = 1}^{d_Z} \left|(\bP \bm\Sigma)_{jj}\right|$ and $\bD_{\text{sign}} = \mathrm{sign}(\bm\Sigma)$.

\subsection{Justifying the sparse intervention approximation} \label{app-sparse-intervention}
The sparse intervention approximation becomes accurate in the large-sample regime.  By Bayes' rule, we have 
\begin{equation}
\begin{aligned}
f_{\ba_e}(y_j \mid \yb_{\pa(j)}) = \int_{\R^{|\pa(j)| + 1}} 
\rmp(y_j \mid z_j)\rmp_{a_{ej}}(z_j \mid z_{\pa(j)}) 
\rmp_{\ba_e}(z_{\pa(j)} \mid \yb_{\pa(j)}) \dd z_{\pa(j) \cup \{j\}}.
\end{aligned}
\end{equation}
As we observe samples $\{\yb_{ei\pa(j)}\}_{i=1}^{N_e}$, when $N_e$ is large, the posterior $\rmp_{\ba_e}(z_{\pa(j)} \mid \yb_{\pa(j)})$ which is invariant across $i$ centers around the distribution $$\widehat \rmp(z_{\pa(j)} \mid \yb_{\pa(j)}) \propto \rmp(\yb_{\pa(j)} \mid \widehat z_{\pa(j)})\, \widehat \rmg(z_{\pa(j)})$$ where $\widehat \rmg$ is the empirical distribution of the denoised estimates $\{\widehat z_{ei\pa(j)}\}_{i=1}^{N_e}$ learned from the normal means model. This substitute yields the approximation $$f_{\ba_e}(y_j \mid \yb_{\pa(j)}) \approx \int_{\R^{|\pa(j)| + 1}} \rmp(y_j \mid z_j)\, \rmp_{a_{ej}}(z_j \mid z_{\pa(j)})\, \widehat \rmp(z_{\pa(j)} \mid \yb_{\pa(j)})\, \dd z_{\pa(j) \cup \{j\}} =: f_{a_{ej}}(y_j \mid \yb_{\pa(j)}),$$ thus asymptotically the conditional density depends on the environment only through the local intervention indicator $a_{ej}$. Thus, the constraint to depend on $a_j$ can be viewed as a causal invariance constraint that requires $\by$ to act in a sparse way given the intervention $\ba$ \citep{yao2025unifying}.

\section{Proof of Thm.~\ref{thm:score-matching-minimizer}}
\solutionstructure*
\begin{proof}
\citet[Theorem~1]{hyvarinen2005estimation} shows that
\begin{equation}\label{eq:sm-ibp}
L^{\cG}\left( s \right)
=
\sum_{e = 1}^M \EE{\by \sim f_{\ba_e}^{\cG}}{
\left\|s(\by, \ba_e) \right\|^2
+ 2\,\mathrm{tr}\!\left( \nabla_{\by}s(\by, \ba_e) \right)
+ \left\| s^{\cG}(\by,\ba_e) \right\|^2 }.
\end{equation}
The last term in~\eqref{eq:sm-ibp} does not depend on $s$ and can be dropped. Hence, minimizing $\mathbf{L}^{\cG}(s)$ is equivalent to minimizing
\begin{equation}\label{eq:sm-reduced}
\sum_{e = 1}^M \EE{\by \sim f_{\ba_e}^{\cG}}{
\left\|s(\by, \ba_e) \right\|^2
+ 2\,\mathrm{tr}\!\left( \nabla_{\by}s(\by, \ba_e) \right)}.
\end{equation}

Using the structure~\eqref{s-function}, we can rewrite~\eqref{eq:sm-reduced} as
\begin{equation}\label{eq:sm-expanded}
s \mapsto
\sum_{e=1}^M\sum_{j=1}^{d_Z}\EE{\by\sim f_{\ba_e}^{\cG}}{
\left|s_j\!\left(y_k,\yb_{\pa(k)},\ba_e\right)\right|^2
+2\,\partial_j s_j\!\left(y_j,\yb_{\pa(j)},\ba_e\right)}.
\end{equation}
\end{proof}

\section{Weighted \texorpdfstring{$f$}{f}-modeling CRL}
 The current method assigns uniform importance to each environment in the objective. However it is often desirable to learn a causal model that prioritizes accuracy in certain environments over others. For example, the base environment $\ba = 0$ may be of particular interest. To account for such preferences, we could consider the weighted score-matching objective:
\begin{equation*}
\widehat{s} = \arg \min_{s \in \cS} \sum_{e = 1}^M w_e \sum_{i = 1}^{N_e} \sum_{j = 1}^{d_Z} \left[ \partial_j [s(\by_{ei}, \ba_e)]_j + \frac{1}{2} \left([s(\by_{ei}, \ba_e)]_j\right)^2 \right],
\end{equation*}
where $w_e \geq 0$ denotes the weight assigned to environment $e$, subject to $\sum_{e=1}^M w_e = 1$. The weights determine the relative influence of each environment on the score matching loss. Higher $w_e$ assign more importance to environment $e$. 

In the absence of prior information or labeled test data, one may choose uniform weights $w_e = 1/M$, or alternatively set $w_e$ proportional to the sample size $N_e$ to account for imbalance across environments. If a particular environment (e.g., the base SCM with $\ba_e = 0$) is of interest, one may assign it a dominant weight (e.g., $w_e \approx 1$) and downweight the others accordingly. Other methods for choosing the environment weights are discussed \citep{Shen2025}.

\section{Details of the Empirical Study and Additional Results}
\label{simulation-details}

\subsection{Synthetic study}
\label{app:synthetic}

{\color{black}
\textbf{Synthetic data generation.} We
generate synthetic data from an SCM over $d_Z=4$ latent variables arranged in a chain $\bz_1 \to \bz_2 \to \bz_3 \to \bz_4$.
For each node $j\in[d_Z]$, the structural equation is
\begin{equation}
\textstyle
\label{eq:SCM_synthetic}
    z_j := \sum_{k\in\pa(j)} w_{jk}\,g(z_k) + u_j,
\end{equation}
\looseness-1 with $w_{jk}\!\iid\!\textsc{Unif}[-1,1]$, noise $u_j \!\iid\!\cN(0,\sigma_z^2)$ with ${\sigma_z^2\!=\!4}$, and nonlinearity $g(z)\!=\!\tanh(\kappa z)\!+\!(\kappa z)^3$ with $\kappa\!=\!3$.
\looseness-1 Each domain results from a single-node
intervention with targets encoded by a one-hot vector $\ba\in\{0,1\}^{d_Z}$ such that an intervention on $z_j$ replaces the corresponding structural equation from~\eqref{eq:SCM_synthetic} with $z_j:=\tilde{u}_j\sim \cN(10,1)$.
The latents $\bz\!\in\!\R^{d_Z}$ are mapped to observations $\bx\!\in\!\R^{d_X}$ with $d_X\!=\!100$ via the linear measurement model from~\eqref{eq:multi_domain_linear_measurements}, where
$\bA^\star\!\in\!\cO^{d_X\times d_Z}$ is a random column-orthonormal matrix and $\sigma^{2\star}\!=\!2$.
\par}

For the synthetic experiments, we run CRL $f$-modeling (\Cref{alg:iterative-SM}) for $1{,}000$ iterations on an H100 GPU. We set $\eta=1$ when using empirical Bayes shrinkage, parameterize each score component $s_j$ by cubic splines with $8$ knots on $[-15,15]$, and initialize $\bA$ using the PCA loading matrix.

In the first experiment (left column of \Cref{fig:main-results}), we simulate $40$ independent runs of the specified causal model by drawing $\bA^\star$ and $\sigma^{2\star}$ anew in each run. In each run, we generate $N_e=2{,}000$ samples for each single-node intervention target.

For the per-environment comparison in \Cref{fig:per-env}, we simulate $M=20$ single-node interventions under a fixed specification $(\bA^\star,\sigma^{2\star})$, with $4$ environments per intervention and $N_e=500$ samples per environment.

\begin{figure}[tbp]
\centering
\begin{minipage}[t]{0.485\textwidth}
\centering
{\small (a)}\par
\includegraphics[width=\linewidth]{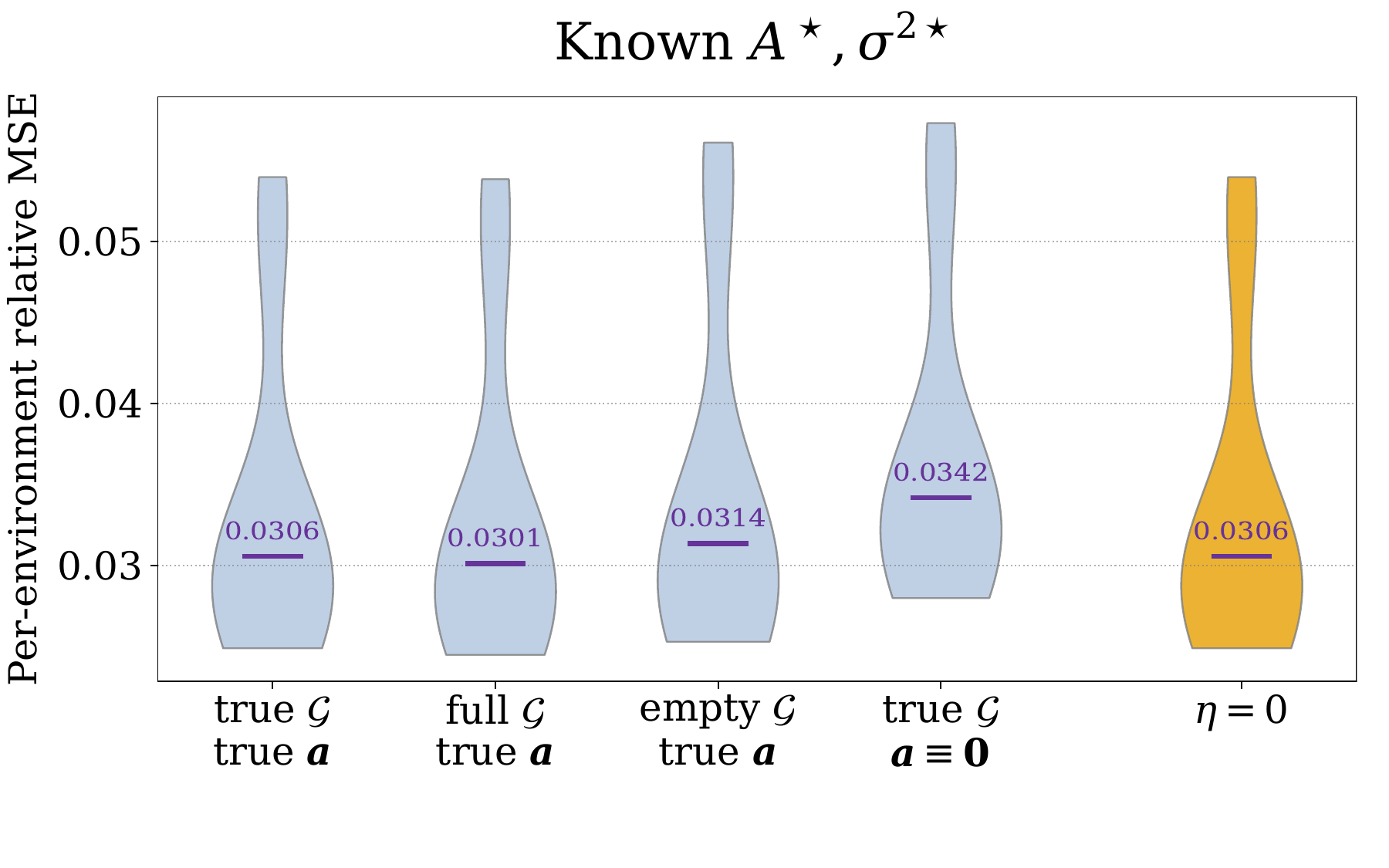}
\end{minipage}\hfill
\begin{minipage}[t]{0.485\textwidth}
\centering
{\small (b)}\par
\includegraphics[width=\linewidth]{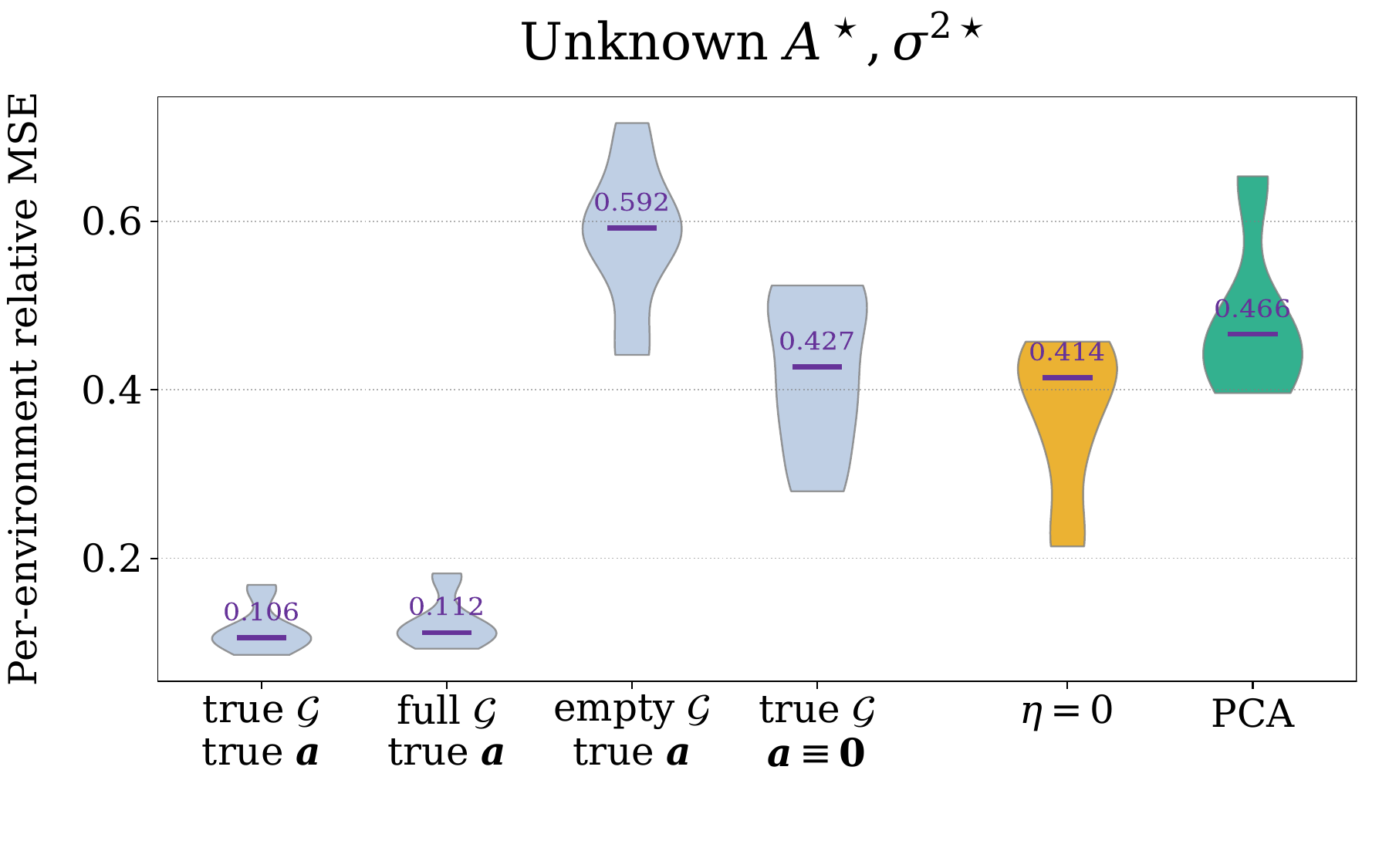}
\end{minipage}
\caption{\textbf{Synthetic per-environment relative MSE.} Results are shown across interventions when (a) the true mixing matrix $\bA^\star$ and noise variance $\sigma^{2\star}$ are held fixed and (b) the measurement model is estimated jointly. The true-DAG CRL $f$-modeling fit has the lowest error in the learned setting and remains comparatively uniform across environments. Lower values indicate better latent recovery.}
\label{fig:per-env}
\end{figure}

\paragraph{Sensitivity analyses.}
We vary the latent dimension, per-environment sample size, measurement noise, and damping factor in \Cref{fig:synthetic-sensitivity}. The dimension-scaling comparisons show similar trends across sample sizes and noise levels. Larger damping values reduce the mixing error in this experiment.

\begin{figure}[!htbp]
\centering
\begin{subfigure}[t]{0.32\linewidth}
\centering\vspace{0pt}
{\small (a)}\par
\includegraphics[width=\linewidth]{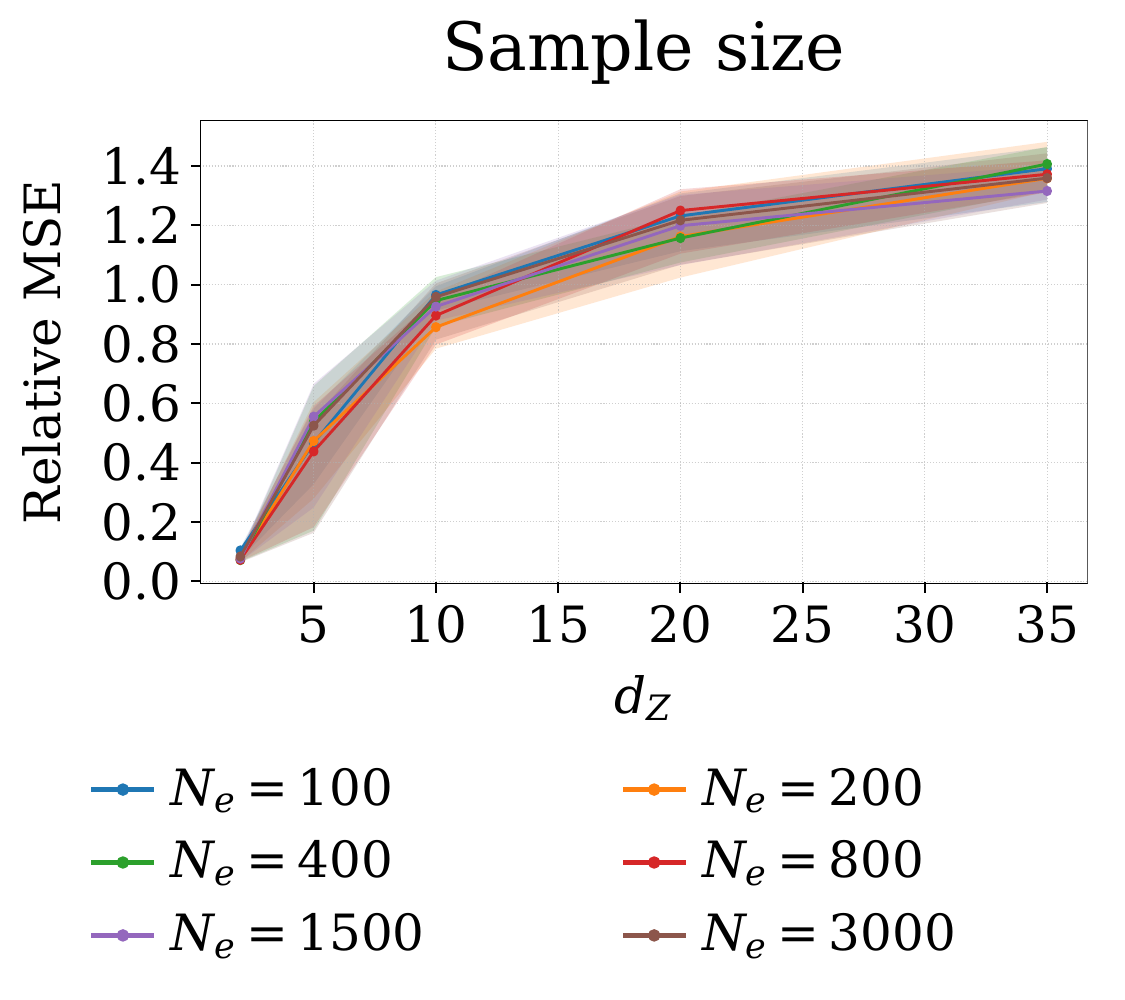}
\phantomsubcaption\label{fig:relmse_vs_dZ_by_Ne}
\end{subfigure}\hfill
\begin{subfigure}[t]{0.32\linewidth}
\centering\vspace{0pt}
{\small (b)}\par
\includegraphics[width=\linewidth]{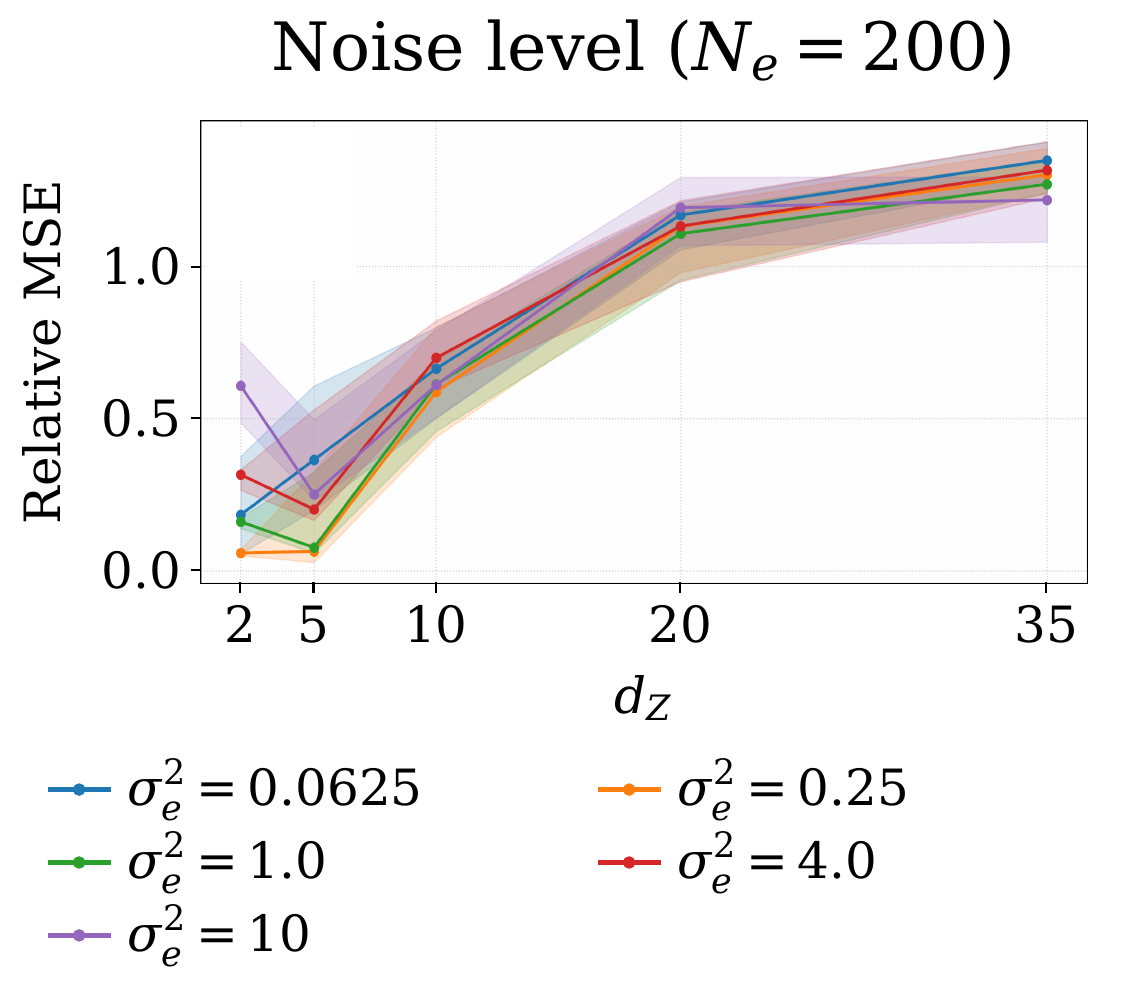}
\phantomsubcaption\label{fig:relmse_vs_dZ_by_sigma2}
\end{subfigure}\hfill
\begin{subfigure}[t]{0.32\linewidth}
\centering\vspace{0pt}
{\small (c)}\par
\includegraphics[width=\linewidth]{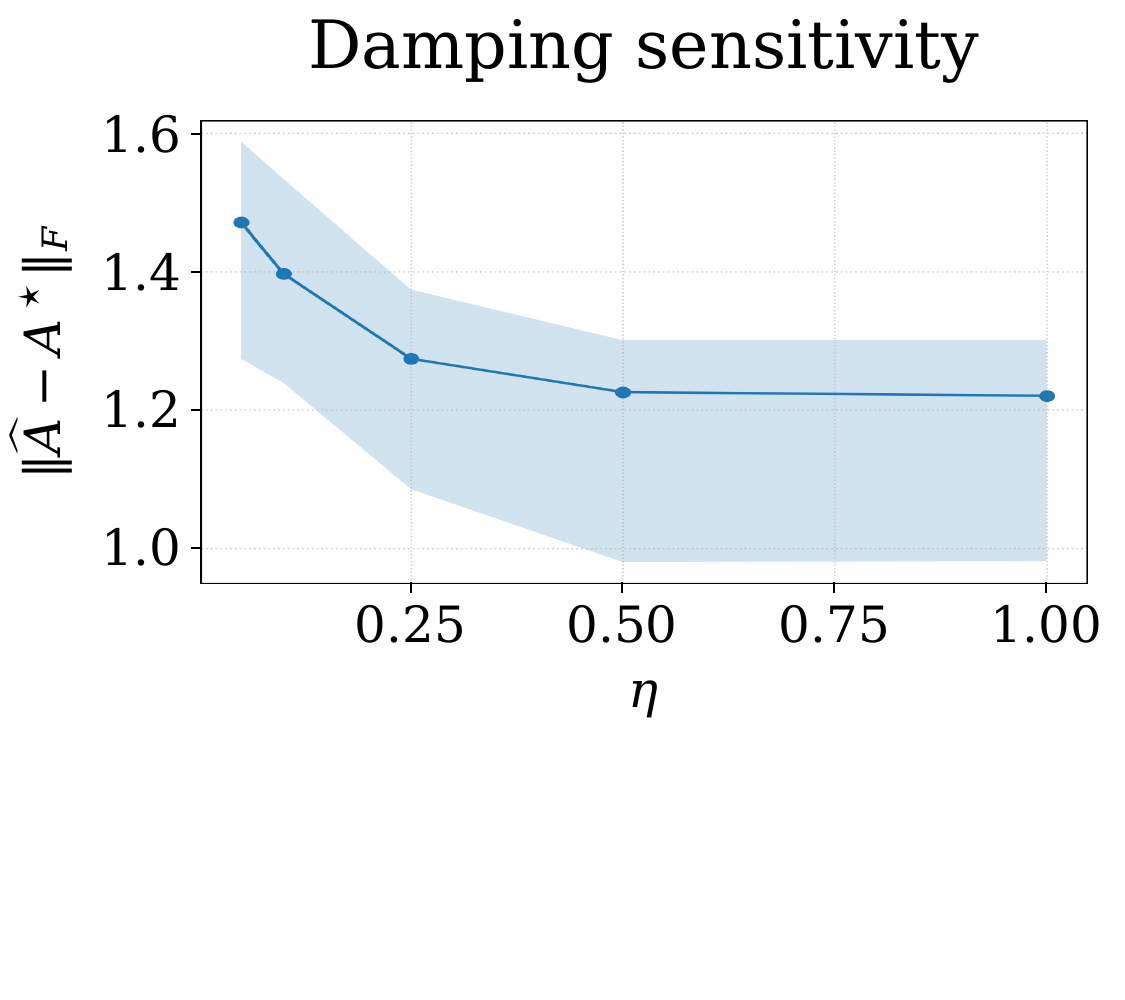}
\phantomsubcaption\label{fig:sensitivity-eta}
\end{subfigure}
\caption{\textbf{Synthetic scaling and damping sensitivity.} Latent relative MSE as dimension $d_Z$ varies across (a) per-environment sample sizes $N_e$ and (b) measurement-noise levels $\sigma^2$ with $N_e=200$. Panel (c) reports mixing-matrix Frobenius error as the damping factor $\eta$ varies. Curves and shaded bands in (a) show medians and interquartile ranges across $20$ runs. Lower values indicate better recovery.}
\label{fig:synthetic-sensitivity}
\end{figure}

\paragraph{Graph and noise variation.}
\Cref{fig:different-DAG} compares graph specifications across measurement-noise levels. The true and causally ordered full graphs generally achieve the lowest errors. To assess robustness to intervention noise, we vary the scale in $\textrm{do}(z_j\sim\mathcal N(10,\sigma_{\mathrm{int}}^2))$ in \Cref{fig:relmse_vs_int_noise}.

\begin{figure}[!htbp]
\centering
\begin{minipage}[t]{0.485\linewidth}
\centering\vspace{0pt}
{\small (a)}\par
\includegraphics[width=\linewidth]{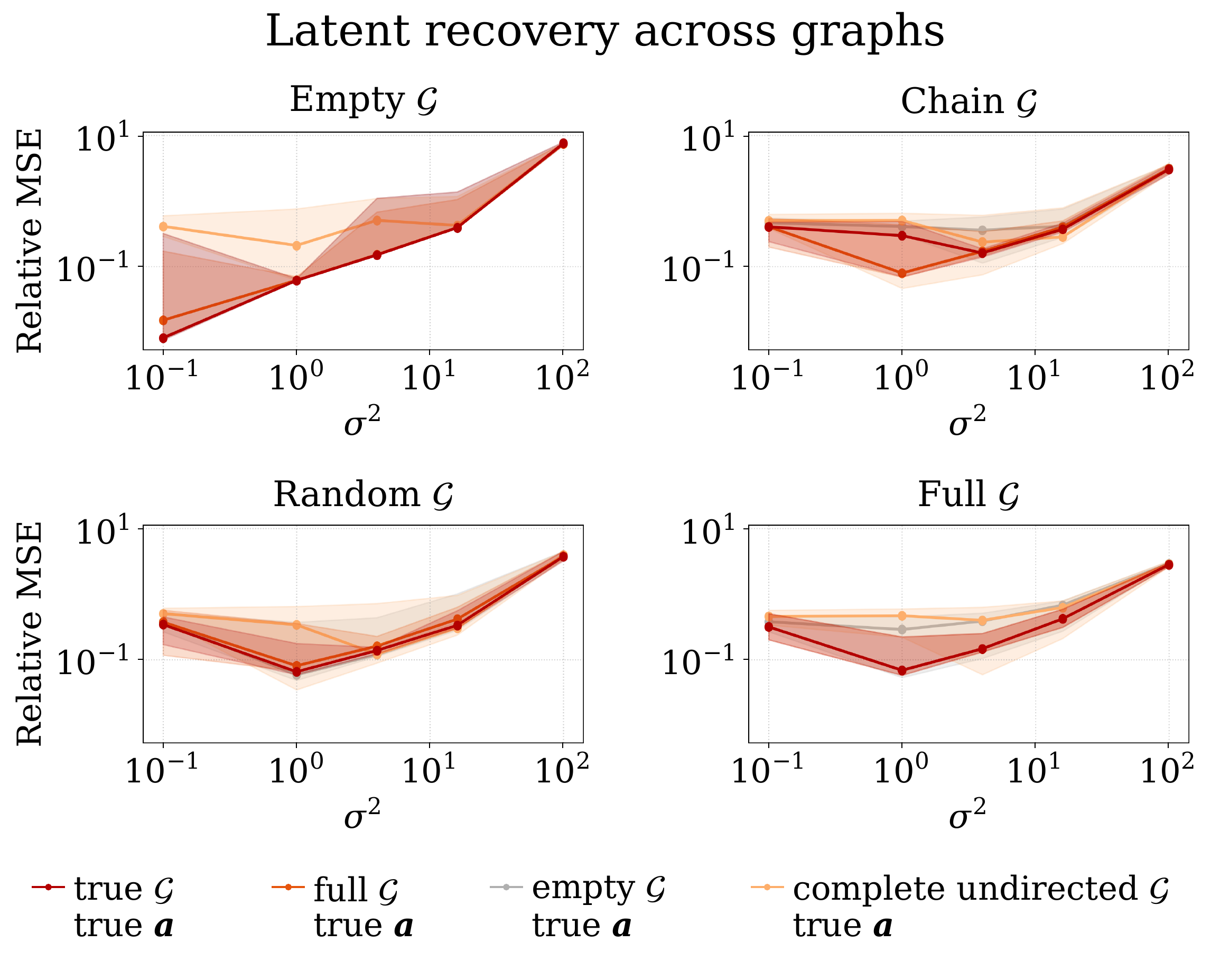}
\end{minipage}\hfill
\begin{minipage}[t]{0.485\linewidth}
\centering\vspace{0pt}
{\small (b)}\par
\includegraphics[width=\linewidth]{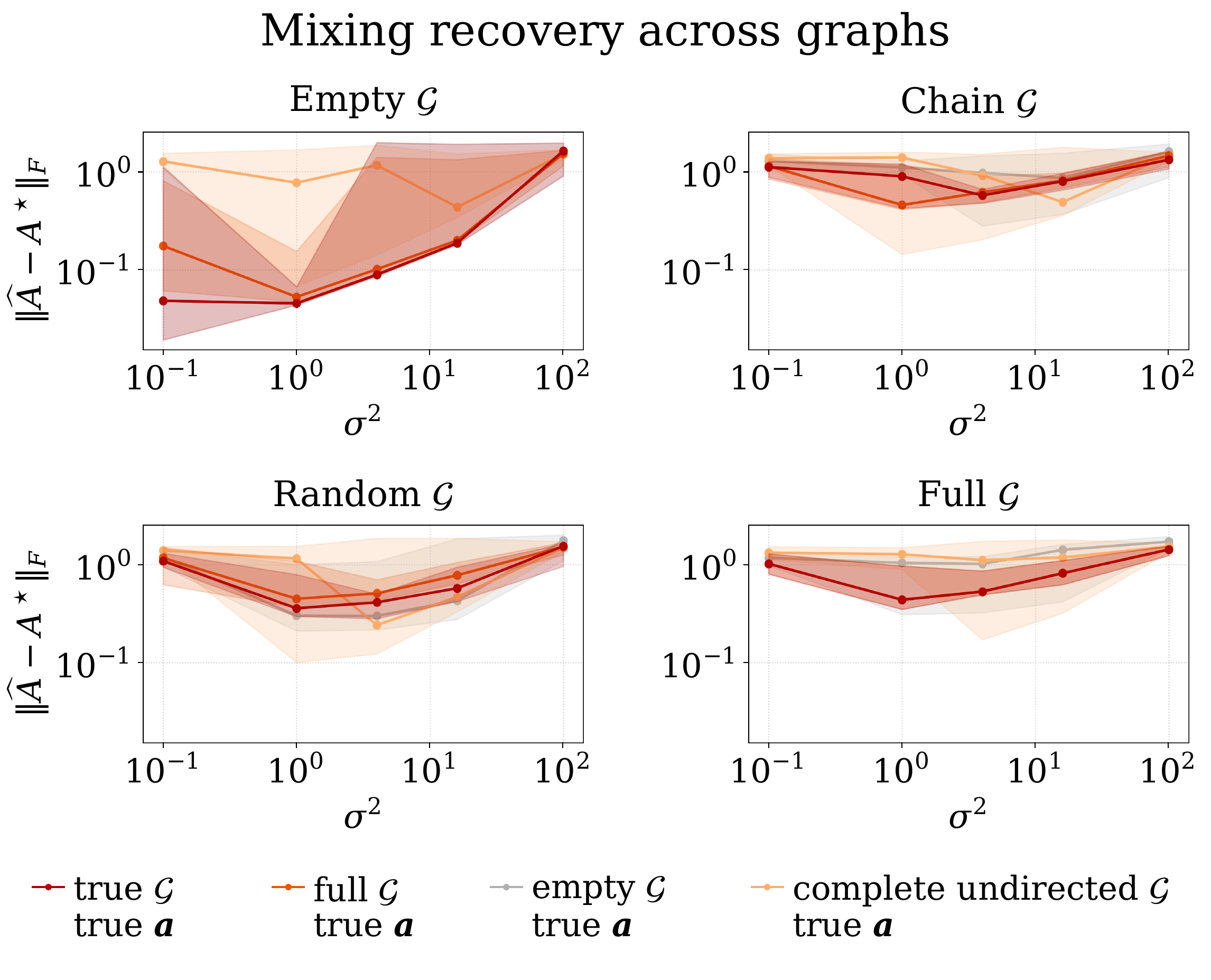}
\end{minipage}
\caption{\textbf{Recovery across graph structures and measurement noise.} (a) Latent relative MSE and (b) mixing-matrix Frobenius error for four data-generating graphs and four score-model graph specifications. The true- and full-DAG fits generally achieve the lowest errors. For the chain, random, and full data-generating graphs, these fits achieve their lowest latent errors at intermediate noise levels. At low noise, the smaller Tweedie correction may limit the benefit of shrinkage.}
\label{fig:different-DAG}
\end{figure}

\begin{figure}[!htbp]
\centering
\includegraphics[width=0.485\linewidth]{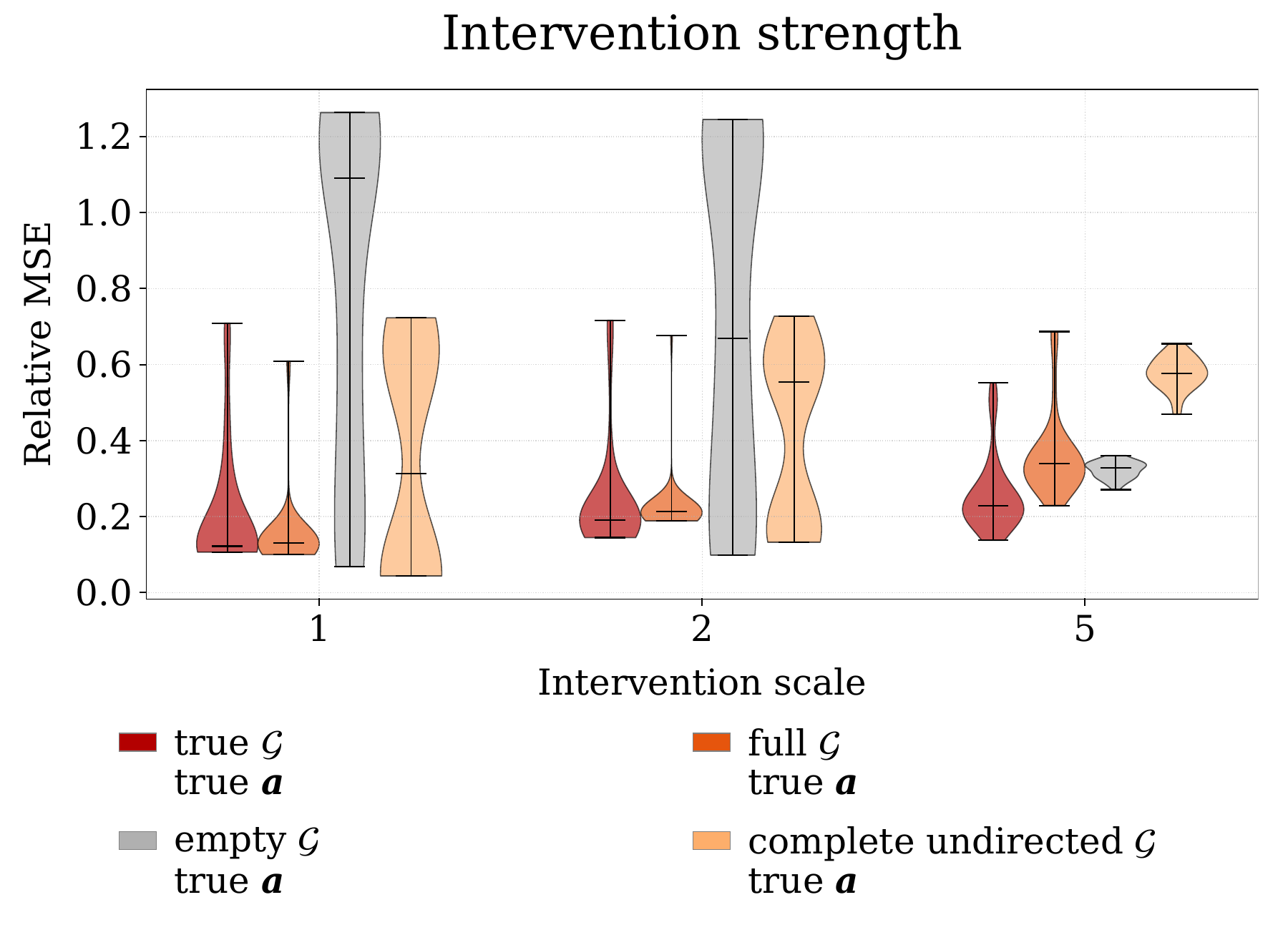}
\caption{\textbf{Recovery as intervention noise varies.} Latent relative MSE for the true, full, empty, and complete undirected score-model graphs. The true-DAG fit performs best, followed by the full DAG. Dispersion across runs decreases as the intervention noise increases. Violins and their summary marks retain the original run distributions.}
\label{fig:relmse_vs_int_noise}
\end{figure}

\paragraph{Factor correlations.}
\Cref{fig:simulated_chain_mcc} compares recovered factors with the true latent variables. The true and full DAGs show similar disentanglement. The complete undirected graph also aligns the factors well in this example, alongside the higher variability in its relative MSE.

\begin{figure}[!htbp]
\centering
\begin{minipage}[t]{0.24\linewidth}
\centering\vspace{0pt}
{\small (a)}\par
\includegraphics[width=\linewidth]{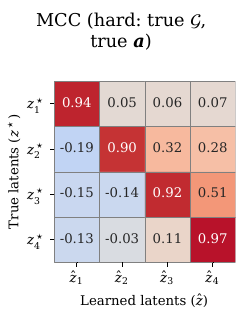}
\end{minipage}\hfill
\begin{minipage}[t]{0.24\linewidth}
\centering\vspace{0pt}
{\small (b)}\par
\includegraphics[width=\linewidth]{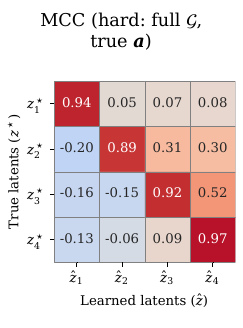}
\end{minipage}\hfill
\begin{minipage}[t]{0.24\linewidth}
\centering\vspace{0pt}
{\small (c)}\par
\includegraphics[width=\linewidth]{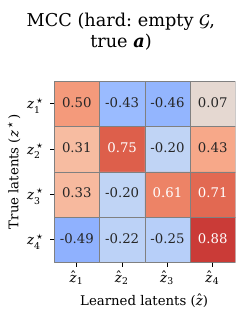}
\end{minipage}\hfill
\begin{minipage}[t]{0.24\linewidth}
\centering\vspace{0pt}
{\small (d)}\par
\includegraphics[width=\linewidth]{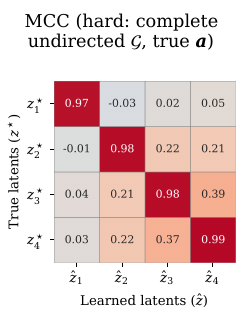}
\end{minipage}
\caption{\textbf{Recovered-factor correlations in the synthetic chain.} Correlations between the true latent variables and estimated representations using (a) the true DAG, (b) the full DAG, (c) the empty graph, and (d) the complete undirected graph. The true and full DAGs yield similar factor alignment.}
\label{fig:simulated_chain_mcc}
\end{figure}

\FloatBarrier
\subsection{\texorpdfstring{A semisynthetic application to single-cell perturbation genomics}{Application to single-cell perturbation genomics}}
\label{sec:application}

The motivating setting for multi-domain CRL is single-cell perturbation genomics. In a Perturb-seq experiment \citep{dixit2016perturb,replogle2022mapping}, each cell receives a genetic perturbation with a known target and is then profiled by single-cell RNA sequencing, so a dataset comprises many domains that share a measurement process and differ by a sparse intervention on a shared regulatory network. We evaluate CRL $f$-modeling on data of this form.

\paragraph{Data.}
We take the latent causal graph to be the myeloid transcription-factor network of \citet{krumsiek2011hierarchical}, which encodes the differentiation of myeloid progenitors through $d_Z=11$ transcription factors, among them \textsc{Gata1}, \textsc{Gata2}, \textsc{PU.1} and \textsc{Cebpa}. We orient the network into a DAG along the differentiation hierarchy~(\cref{fig:krumsiek-dag}), which gives structural equations of the form~\eqref{eq:SCM_synthetic} in the Appendix. The reported activation and repression signs are kept on the edges. Each domain corresponds to a CRISPRi knockdown of one transcription factor, modelled as a hard intervention on its known target, and one further domain is left unperturbed. The latent factor activities are mapped to $d_X=100$ expression readouts through a column-orthonormal mixing matrix with Gaussian measurement noise, as in~\eqref{eq:multi_domain_linear_measurements}, and we draw $N_e=800$ cells per domain. The data are simulated from the model, so the latent factors and the mixing matrix are known. The study is semi-synthetic, real in its regulatory network and interventional design and model-induced in its measurements.

\begin{figure}[tbp]
\centering
\begingroup
\definecolor{bioDagOrange}{HTML}{D97706}
\tikzset{
  bio node/.style={draw=black!60, fill=black!4, line width=0.5pt,
    rounded corners=2pt, minimum width=0.95cm, minimum height=0.46cm,
    inner xsep=3pt, inner ysep=2pt, outer sep=1.2pt,
    font=\rmfamily\fontsize{8.5}{10}\selectfont, text=black},
  bio target/.style={bio node, fill=bioDagOrange!23,
    draw=black!85, line width=1.1pt},
  bio edge/.style={-stealth, draw=black!70, line width=0.6pt},
  bio act/.style={bio edge},
  bio rep/.style={draw=black!80, line width=0.7pt, dashed, -|, shorten >=3pt},
  bio key/.style={font=\rmfamily\fontsize{8}{9}\selectfont,
    text=black, inner sep=0pt}
}
\captionsetup[subfigure]{font=small, skip=5pt, justification=centering, position=top}
\begin{subfigure}[t]{0.49\linewidth}
\centering
\vspace{0pt}
\caption{Myeloid gene-regulatory network}
\label{fig:krumsiek-dag}
\begin{tikzpicture}[x=1cm,y=1cm]
  \path[use as bounding box] (-4,-0.55) rectangle (4,4.70);
  \node[bio node] (gata2) at (0,4.35) {Gata2};
  \node[bio node] (gata1) at (-1.80,3.05) {Gata1};
  \node[bio node] (pu1)   at (2.20,3.05) {PU.1};
  \node[bio node] (fog1)  at (-3.20,1.55) {Fog1};
  \node[bio node] (eklf)  at (-1.95,1.55) {EKLF};
  \node[bio node] (fli1)  at (-0.70,1.55) {Fli1};
  \node[bio node] (scl)   at (0.55,1.55) {SCL};
  \node[bio node] (cebpa) at (1.88,1.55) {Cebpa};
  \node[bio node] (cjun)  at (3.12,1.55) {cJun};
  \node[bio node] (gfi1)  at (1.88,0.45) {Gfi1};
  \node[bio node] (egr)   at (3.12,0.45) {EgrNab};
  \draw[bio act] (gata2) -- (gata1);
  \draw[bio act] (gata2) -- (pu1);
  \draw[bio act] (gata1) -- (fog1);
  \draw[bio act] (gata1) -- (eklf);
  \draw[bio act] (gata1) -- (fli1);
  \draw[bio act] (gata1) -- (scl);
  \draw[bio act] (pu1) -- (cebpa);
  \draw[bio act] (pu1) -- (cjun);
  \draw[bio act] (pu1.east) .. controls (3.95,3.05) and (3.95,0.45) .. (egr.east);
  \draw[bio act] (cjun) -- (egr);
  \draw[bio act] (cebpa) -- (gfi1);
  \draw[bio rep] (pu1) -- (scl);
  \draw[bio act] (-2.35,-0.30) -- (-1.80,-0.30);
  \node[bio key,anchor=west] at (-1.65,-0.30) {Activation};
  \draw[bio rep] (0.35,-0.30) -- (0.90,-0.30);
  \node[bio key,anchor=west] at (1.05,-0.30) {Repression};
\end{tikzpicture}
\end{subfigure}\hfill
\begin{subfigure}[t]{0.49\linewidth}
\centering
\vspace{0pt}
\caption{Sachs protein-signalling network}
\label{fig:sachs-dag}
\begin{tikzpicture}[x=1cm,y=1cm]
  \path[use as bounding box] (-4,-0.55) rectangle (4,4.70);
  \node[bio node]   (plcg) at (-3.10,3.25) {Plcg};
  \node[bio target] (pip3) at (-3.10,4.35) {PIP3};
  \node[bio target] (pip2) at (-1.45,4.35) {PIP2};
  \node[bio target] (pkc)  at (0,3.25) {PKC};
  \node[bio target] (pka)  at (0,1.25) {PKA};
  \node[bio node]   (raf)  at (1.55,2.90) {Raf};
  \node[bio target] (mek)  at (3.10,3.25) {Mek};
  \node[bio node]   (erk)  at (3.10,1.25) {Erk};
  \node[bio target] (akt)  at (1.80,0.25) {Akt};
  \node[bio node]   (p38)  at (-2.00,2.25) {p38};
  \node[bio node]   (jnk)  at (-0.90,2.25) {Jnk};
  \draw[bio edge] (plcg) -- (pip3);
  \draw[bio edge] (plcg) -- (pip2);
  \draw[bio edge] (plcg) -- (pkc);
  \draw[bio edge] (pip3) -- (pip2);
  \draw[bio edge] (pip3.west)
    .. controls (-3.85,4.35) and (-3.85,4.05) .. (-3.85,3.55)
    -- (-3.85,0.75)
    .. controls (-3.85,0.25) and (-3.55,0.25) .. (-3.05,0.25)
    -- (akt.west);
  \draw[bio edge] (pip2) -- (pkc);
  \draw[bio edge] (pkc) -- (pka);
  \draw[bio edge] (pkc) -- (raf);
  \draw[bio edge] (pkc) -- (mek);
  \draw[bio edge] (pkc) -- (p38);
  \draw[bio edge] (pkc) -- (jnk);
  \draw[bio edge] (pka) -- (raf);
  \draw[bio edge] (pka) -- (mek);
  \draw[bio edge] (pka) -- (erk);
  \draw[bio edge] (pka) -- (akt);
  \draw[bio edge] (pka) -- (p38);
  \draw[bio edge] (pka) -- (jnk);
  \draw[bio edge] (raf) -- (mek);
  \draw[bio edge] (mek) -- (erk);
  \draw[bio edge] (erk) -- (akt);
  \node[bio target,minimum width=0.36cm,minimum height=0.22cm,
    inner sep=0pt] at (-2.45,-0.30) {};
  \node[bio key,anchor=west] at (-2.10,-0.30) {Targeted};
  \node[bio node,minimum width=0.36cm,minimum height=0.22cm,
    inner sep=0pt] at (0.35,-0.30) {};
  \node[bio key,anchor=west] at (0.70,-0.30) {Other protein};
\end{tikzpicture}
\end{subfigure}
\caption{\textbf{Biological directed acyclic graphs used in the semisynthetic studies.}
  (a)~The 11-node myeloid gene-regulatory network of
  \citet{krumsiek2011hierarchical}, oriented along the differentiation hierarchy.
  Solid arrows denote activation and the dashed T-bar denotes repression
  from PU.1 to SCL.
  (b)~The 11-node consensus protein-signalling network of
  \citet{sachs2005causal}, oriented as an acyclic proxy for a system with feedback.
  Orange fill and heavy outlines identify the six assigned intervention targets:
  Akt, PKC, PIP2, Mek, PIP3 and PKA.
  The seven targeted conditions include two assigned to PKC
  (\Cref{tab:sachs-interventions}).}
\label{fig:biological-dags}
\endgroup
\end{figure}

\paragraph{Results.}
We compare the true DAG, full DAG, empty DAG, pooled fit, no shrinkage ($\eta=0$) and PCA over $20$ paired replicates with learned mixing. The true-, empty-, and full-DAG fits retain initialization from interventional mean shifts. The no-shrinkage and pooled fits use PCA initialization after feature standardization; the direct PCA baseline uses centered, unstandardized observations. We run the iterative fits for $3{,}000$ steps with six spline knots. The score-based variants use damping $0.9$ within the iterations and a full-strength Tweedie correction for the final latent estimate. The $\eta=0$ variant continues to update $\bA$ and returns the latent projection.

With PCA initialization, the no-shrinkage variant attains median latent relative MSE $1.046$ and mixing-matrix relative Frobenius error $0.831$~(\cref{fig:app-learned}), compared with $1.049$ and $0.756$ for PCA. The true-DAG medians are $0.796$ and $0.599$, close to the empty-DAG medians of $0.792$ and $0.591$; both fits have lower median errors than the other four methods. These comparisons reflect differences in initialization as well as fitting procedure. The results use the original simulator, which generates descendants before overwriting intervention targets, and therefore serve as diagnostics of that simulation protocol.

The true-DAG fit has mean factor correlation $0.60$~(\cref{fig:app-mcc-learned}), reflecting residual error in its estimated mixing. With the measurement model known, causal score modeling and empirical Bayes shrinkage improve latent recovery, and the true-DAG fit attains mean factor correlation $0.97$ (\cref{app:known-mixing}).

\begin{figure}[tbp]
\centering
\begin{subfigure}[t]{0.485\textwidth}
\centering
\vspace{0pt}
{\small (a)}\par
\includegraphics[width=\linewidth]{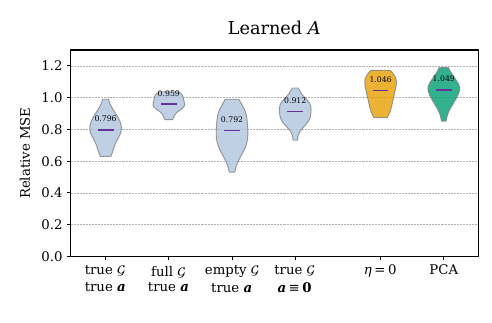}
\phantomsubcaption
\label{fig:krumsiek-z-recovery}
\end{subfigure}
\hfill
\begin{subfigure}[t]{0.485\textwidth}
\centering
\vspace{0pt}
{\small (b)}\par
\includegraphics[width=\linewidth]{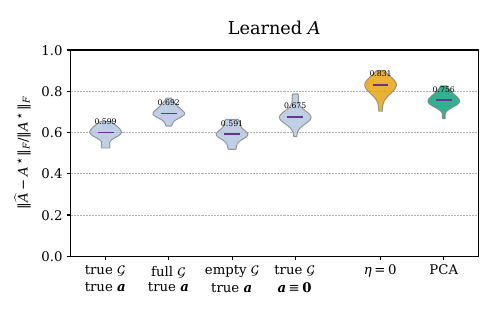}
\phantomsubcaption
\label{fig:krumsiek-A-recovery}
\end{subfigure}
\caption{\textbf{Semisynthetic single-cell perturbation results with learned mixing.} On the myeloid network in \Cref{fig:krumsiek-dag}, we compare the true DAG, full DAG, empty DAG, pooled fit, no shrinkage ($\eta=0$) and PCA using (a) latent-factor relative MSE and (b) mixing-matrix relative Frobenius error. The no-shrinkage variant uses PCA initialization and continues to learn $\bA$; the true-, empty-, and full-DAG fits retain mean-shift initialization. The true-DAG and empty-DAG fits have the lowest median errors. Violins include all $20$ paired replicates; horizontal marks and labels show medians. Lower values indicate better recovery. The original simulation protocol and initialization differences are detailed in \Cref{sec:application}.}
\label{fig:app-learned}
\end{figure}

\subsubsection{Known-mixing results}
\label{app:known-mixing}
When the linear measurement model is taken as known, the perturbation study of \cref{sec:application} isolates the contribution of the causal score model. CRL $f$-modeling with the true DAG and with the complete DAG attain the lowest relative MSE, $0.042$ and $0.039$, the empty graph is worse at $0.045$, the pooled fit worse still at $0.047$, and projection without shrinkage~($\eta=0$) is worst at $0.062$~(\cref{fig:app-known}a). The error is uniform across the eleven knockdowns~(\cref{fig:app-known}b).

\begin{figure}[!htbp]
\centering
\includegraphics[width=\linewidth]{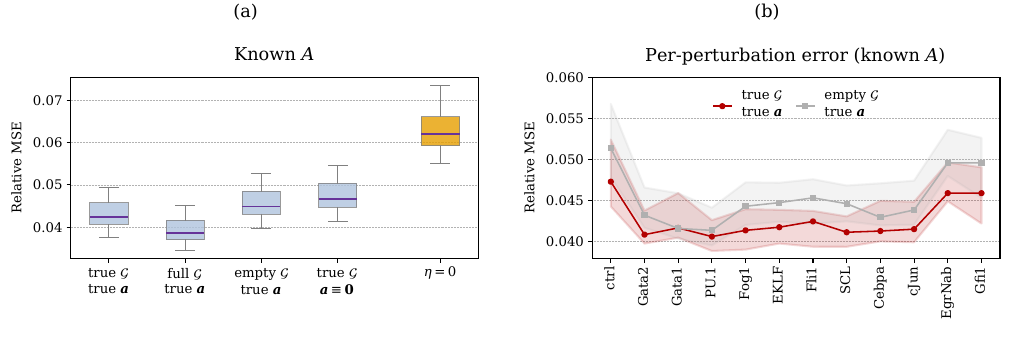}
\caption{\textbf{Known-mixing results for the perturbation study.} (a)~With the mixing matrix known, using the causal graph and empirical Bayes shrinkage lowers the relative MSE of the latent factors, and projection without shrinkage~($\eta=0$) is worst. (b)~The per-perturbation error is uniform across the eleven knockdowns. Boxes and bands span the interquartile range over $20$ replicates.}
\label{fig:app-known}
\end{figure}

\subsubsection{Recovered-factor correlations}
\label{app:learned-mcc}
Factor recovery reflects the accuracy of the mixing matrix. With the mixing matrix known, the recovered latent factors align closely with the true transcription factors, at a mean correlation of $0.97$~(\cref{fig:app-mcc}). In the learned-mixing setting of \cref{sec:application} recovery is only partial, at a mean correlation of $0.60$~(\cref{fig:app-mcc-learned}), reflecting the residual error in the estimated mixing.

\begin{figure}[!htbp]
\centering
\begin{subfigure}[t]{0.485\linewidth}
\centering\vspace{0pt}
{\small (a)}\par
\includegraphics[width=\linewidth]{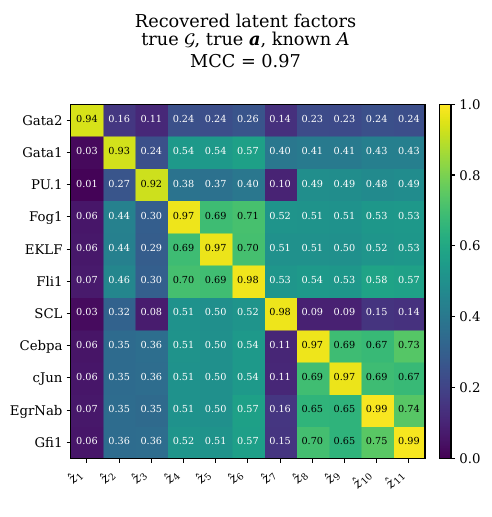}
\phantomsubcaption
\label{fig:app-mcc}
\end{subfigure}\hfill
\begin{subfigure}[t]{0.485\linewidth}
\centering\vspace{0pt}
{\small (b)}\par
\includegraphics[width=\linewidth]{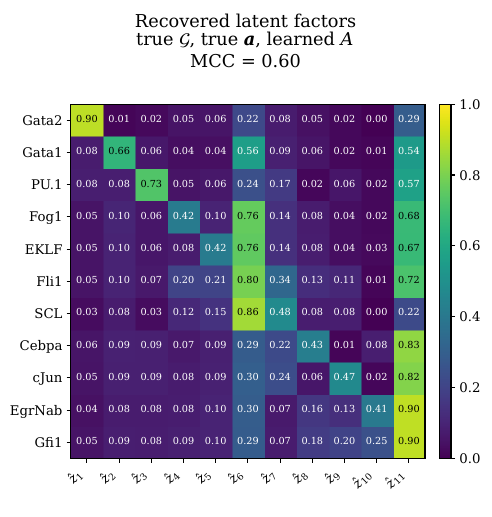}
\phantomsubcaption
\label{fig:app-mcc-learned}
\end{subfigure}
\caption{\textbf{Recovered latent factors versus the transcription factors.} Absolute correlation between each estimated latent $\hat\bz_k$ and the true transcription-factor activities. (a)~With the mixing matrix known, recovery is near-perfect, at a mean correlation of $0.97$. (b)~With the mixing matrix estimated jointly, recovery is partial, at a mean correlation of $0.60$, reflecting the residual error in the estimated mixing. Off-diagonal blocks in both settings reflect the erythroid and myeloid branches of the network.}
\label{fig:app-mcc-both}
\end{figure}

\FloatBarrier

\subsection{\texorpdfstring{Details of the single-cell protein-signalling study}{Details of the single-cell protein-signalling study}}
\label{app:sachs}
The study of \cref{sec:sachs} uses the single-cell protein-signalling data of \citet{sachs2005causal}, comprising $N=7466$ cells in nine conditions and $d_Z=11$ signalling molecules (\textsc{Raf}, \textsc{Mek}, \textsc{Plcg}, \textsc{PIP2}, \textsc{PIP3}, \textsc{Erk}, \textsc{Akt}, \textsc{PKA}, \textsc{PKC}, \textsc{p38}, \textsc{Jnk}). We use the measured-condition CSVs in the public \href{https://zenodo.org/records/7681811}{data archive}, apply $\log(1+x)$ to each measurement, and standardise each protein across pooled cells, preserving interventional shifts. The latent graph is the consensus signalling network of \citet{sachs2005causal}, oriented into a DAG~(\cref{fig:sachs-dag}). \Cref{tab:sachs-interventions} lists the conditions and target labels used in the experiment. The seven targeted conditions cover six proteins. Both G0076 and PMA are assigned to \textsc{PKC}, using the same target indicator and the same column of the shared mixing matrix. The current score implementation pools all intervened coordinates into one shared score component, including these two \textsc{PKC} conditions.

\begin{center}
\begin{minipage}{\linewidth}
\centering
\captionof{table}{Sachs conditions and assigned intervention targets. The two general stimulation conditions are modeled as observational. The seven targeted conditions cover six distinct proteins, with \textsc{PKC} targeted twice. File names identify the CSVs used from the data archive.}
\label{tab:sachs-interventions}
\begin{tabular}{lll}
\toprule
Condition & Data file & Assigned target \\
\midrule
CD3/CD28 & \texttt{cd3cd28.csv} & None (observational) \\
CD3/CD28 + ICAM-2 & \texttt{cd3cd28\_icam2.csv} & None (observational) \\
CD3/CD28 + Akt inhibitor & \texttt{cd3cd28\_aktinhib.csv} & \textsc{Akt} \\
CD3/CD28 + G0076 & \texttt{cd3cd28\_g0076.csv} & \textsc{PKC} \\
CD3/CD28 + Psitectorigenin & \texttt{cd3cd28\_psitect.csv} & \textsc{PIP2} \\
CD3/CD28 + U0126 & \texttt{cd3cd28\_u0126.csv} & \textsc{Mek} \\
CD3/CD28 + LY294002 & \texttt{cd3cd28\_ly.csv} & \textsc{PIP3} \\
PMA & \texttt{pma.csv} & \textsc{PKC} \\
$\beta$2cAMP & \texttt{b2camp.csv} & \textsc{PKA} \\
\bottomrule
\end{tabular}
\end{minipage}
\end{center}

In each of $20$ replicates, we independently draw a column-orthonormal $\bA^\star\in\mathcal O^{100\times 11}$ and unit-variance isotropic Gaussian measurement noise, keeping the observed protein profiles and condition labels fixed. We initialise learned $\bA$ by PCA and run \cref{alg:iterative-SM} for $150$ iterations with damping $\eta=0.9$ and six spline knots. The iterative fits initialise from feature-standardised observations, whereas the direct PCA baseline uses centered, unstandardised observations. With known $\bA$, we compare the true DAG, full DAG, empty graph, pooled fit and no-shrinkage projection ($\eta=0$). With learned $\bA$, we compare the true DAG, full DAG, empty DAG, pooled fit, $\eta=0$ and PCA. The $\eta=0$ fit continues to update $\bA$ while omitting score-based shrinkage of $\bz$. Each estimate is aligned to the measured proteins by a signed permutation chosen to maximise absolute factor correlations. The same signed permutation is applied to the mixing columns, and mixing recovery is reported as $\|\widehat\bA-\bA^\star\|_F$. We convert the saved relative errors by multiplying by $\|\bA^\star\|_F=\sqrt{11}$, up to numerical rounding. The consensus DAG is an acyclic approximation to a signalling system with feedback.

The pooled known-$\bA$ fit is numerically unstable in one replicate, producing relative MSE $414.9$; we retain the run in the data and all-run median but omit it from the displayed violin in \Cref{fig:sachs-known}. The learned-$\bA$ full-DAG fit likewise has one high-error replicate, with relative MSE $13.98$; it remains in the all-run median but is omitted from the displayed violin in \Cref{fig:main-results}c. With learned $\bA$, the empty-DAG fit also has the highest median MCC, $0.41$, compared with $0.37$ for the true DAG~(\cref{fig:sachs-mcc}).

\begin{figure}[!htbp]
\centering
\begin{subfigure}[t]{0.32\linewidth}
\centering\vspace{0pt}
{\small (a)}\par
\includegraphics[width=\linewidth]{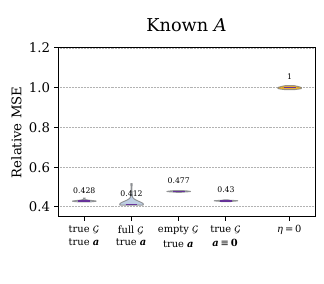}
\phantomsubcaption\label{fig:sachs-known}
\end{subfigure}\hfill
\begin{subfigure}[t]{0.32\linewidth}
\centering\vspace{0pt}
{\small (b)}\par
\includegraphics[width=\linewidth]{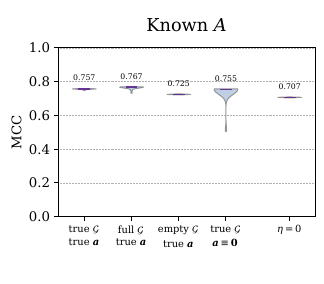}
\phantomsubcaption\label{fig:sachs-known-mcc}
\end{subfigure}\hfill
\begin{subfigure}[t]{0.32\linewidth}
\centering\vspace{0pt}
{\small (c)}\par
\includegraphics[width=\linewidth]{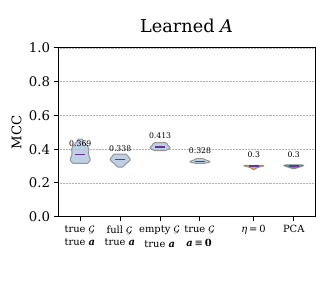}
\phantomsubcaption\label{fig:sachs-learned-mcc}
\end{subfigure}
\caption{\textbf{Additional Sachs recovery results.} (a) Latent relative MSE with known mixing. Mean correlation coefficient (MCC) with (b) known and (c) learned mixing. Violins summarize $20$ replicates, with horizontal purple marks indicating medians. The pooled-fit density in (a) omits the seed-$17$ relative MSE of $414.9$ for legibility, while its median retains all $20$ runs. Lower relative MSE and higher MCC indicate better recovery.}
\label{fig:sachs-supplementary}
\label{fig:sachs-mcc}
\end{figure}

\FloatBarrier

\section{Identification Theory} \label{sect-identifiability}
In this section, we review identifiability results of the latent variables $\bz$ in CRL, with an emphasis on the linear measurement model with nonlinear latent causal model.  Denote a generic estimator $\bz$ given $\bx$ by $\widehat \bz(\bx): \R^{d_X} \mapsto \R^{d_Z}$.

In representation learning, point identification is typically unattainable; instead, one often aims for set identification or identification up to an equivalence class \citep{ahuja2023interventional,yao2025unifying,moran2026towards}. Identification results can be algorithm-agnostic and do not require a specific method to recover the parameters. The working model~\eqref{eq:multi_domain_linear_measurements} corresponds to the setting of causal representation learning (CRL) with a linear measurement model, a general SCM, and known causal DAG and intervention targets. We now review some useful notions of identifiability to discuss our results.

\begin{definition}[Element-identifiability / Disentanglement]
\label{def:disentanglement}
A learned representation $\widehat \bz \in \R^{d_Z}$ is said to be \emph{element-identifiable} if there exists a permutation matrix $\bP_\pi \in \R^{d_Z \times d_Z}$ and an element-wise diffeomorphism $h(\bz) := (h_1(z_1), \ldots, h_{d_Z}(z_{d_Z})) \in \R^{d_Z}$ such that
\begin{equation}
    \widehat \bz = \bP_\pi h(\bz^\star).
\end{equation}
\end{definition}

Element-identifiability is also known as \emph{perfect latent recovery} in the language of \citet{varici2024general}, where a representation $\widehat \bz$ satisfying this condition is said to be \emph{perfect}. It is also referred to as \emph{disentanglement} in the sense of \citet[Definition 5]{moran2026towards}.

\begin{definition}[Scale-permutation-identifiability]
\label{def:scale-permutation-identifiability}
A learned representation $\widehat \bz \in \R^{d_Z}$ is said to be \emph{scale-permutation-identifiable} if there exists a permutation matrix $\bP_\pi \in \R^{d_Z \times d_Z}$ and a diagonal sign matrix $\bD \in \R^{d_Z \times d_Z}$ such that
\begin{equation}
    \widehat \bz = \bP_\pi \bD \bz^\star.
\end{equation}
\end{definition}
Scale-permutation-identifiability is a stronger form of element-identifiability. It describes recovery of the ground-truth latent variables up to permutation and rescaling, which are common forms of ambiguities in linear latent variable models.
\begin{definition}[Mixing-identifiablity]
\label{def:mixing-identifiability}
A learned representation $\widehat \bz \in \R^{d_Z}$ is said to be \emph{mixing-identifiable} if there exists a permutation matrix $\bP_\pi \in \R^{d_Z \times d_Z}$, a diagonal matrix $\bD \in \R^{d_Z \times d_Z}$ and sparse matrix $\bC \in \R^{d_Z \times d_Z}$ satisfying that $\bC_{jj} = 0$ for all $j \in [d_Z]$ and
\begin{equation*}
   \ch(i) \nsubseteq \ch(j) \implies \bC_{ij} = 0, \forall i \in [d_Z], 
\end{equation*}
such that $\widehat \bz = \bP_\pi (\bD + \bC) \bz^\star$. 
\end{definition}
This condition ensures that $\widehat \bz$ is constrained by the causal graph structure, where mixing occurs only among variables with nested parent sets.

There has been extensive recent work on the question of identifiability in CRL. Under \Cref{ass:linear_measurement_model}, identifiability typically holds only up to an equivalence class determined by the available interventions; for example, \citet{squires2023linear} establish such results under a linear measurement model together with a \emph{linear} latent SCM. Several works clarify the limits and possibilities of identifiability. Under hard interventions and a linear measurement model, \citet[Prop.~2]{squires2023linear} and \citet[Remark~2]{Buchholz2023} show that latent variables are identifiable only up to scaling and permutation, and that this is optimal without further assumptions on the causal variables~$\bz$. For stochastic interventions, \citet[Thm.~14]{varici2025score} show that under soft interventions and \Cref{ass:linear_measurement_model}, one can identify latent variables up to mixing with their parents and recover the DAG up to transitive closure; moreover, if each node receives one hard intervention, then under a linear latent SCM both the DAG and the latent variables are identifiable (up to scaling) \citep[Thm.~16]{varici2025score}. For general latent SCMs under soft interventions, \citet[Thm.~18]{varici2025score} show that intervening on each node at least once suffices to recover the latent variables up to parent mixing. Beyond linear SCMs, \citet[Thm.~3]{jin2024learning} establish recovery guarantees for general nonparametric SCMs via score-matching objectives, identifying latents up to componentwise monotone transformations under suitable conditions. In the fully nonparametric measurement setting, \citet{von2023nonparametric} provide identifiability guarantees (and impossibility results) showing that element-identifiability in the sense of \Cref{def:disentanglement} is the sharp guarantee in general, even when one perfect intervention per node is available.

We refer to \citet{varici2025score} for a comprehensive study of identifiability results for multi-domain CRL across linear and nonparametric models and a range of interventional designs. We emphasize, however, that most of the aforementioned CRL identifiability theory does \emph{not} assume the causal graph or intervention targets to be known. By contrast, our approach takes the graph and targets as inputs and focuses on the downstream estimation problem: given a causal factorization, how can one accurately estimate the latent causal variables and the mixing matrix from finite noisy samples? This known-graph regime is studied in the setting of causal component analysis~\citep{wendong2023causal}, which establishes identifiability results for several classes of interventions. Our contribution is complementary: we use the given causal structure to construct an empirical Bayes estimator for finite-sample latent recovery.

\section{\texorpdfstring{$\rmg$}{g}-modeling and related work}
\label{sec:EB_g_modeling}

This appendix develops $\rmg$-modeling for the multi-domain CRL setting and uses this perspective to discuss existing methods that explicitly model the interventional priors. 

As discussed in~\cref{sec:EB_estimation_inference}, $\rmg$-modeling explicitly models the priors $\rmp^\star(\zb\mid\ab_e)$ via a parametric family  ${\{\rmp_\gammab(\zb\mid\ab_e):\gammab\in\Gamma\}}$ and aims to jointly infer the parameters $\thetab:=(\Ab,\sigma^{2},\gammab)$ by maximizing the log marginal likelihood
\begin{equation}
\begin{aligned}
\label{eq:log_marginal_likelihood}
\log \rmp_{\thetab} \left( \cD\right)&=\sum_{e = 1}^M \sum_{i = 1}^{N_e} \log \rmp_{\thetab}(\xb_{ei}\mid\ab_e),\\
\text{where} \quad 
\rmp_{\thetab} (\xb\mid\ab)&=\int\rmp_{\bA,\sigma^2}(\bx \mid \bz) \rmp_\gammab(\zb\mid\ab) d\zb.
\end{aligned}
\end{equation}
\Cref{eq:log_marginal_likelihood} is difficult to optimize exactly because it involves intractable integrals.
However, we can use variational inference (VI) to approximate it~\citep{blei2017variational}.
 For any variational posterior $\rmq_{ei}\in \cP(\R^{d_Z})$ over~$\zb_{ei}$, 
 we can bound the marginal likelihood by the \emph{evidence lower bound} (ELBO), 
\begin{multline*}
\label{eq:lower_bound_to_likelihood}
 \log \rmp_{\thetab} \left(\xb_{ei}\mid\ab_e\right)
 \geq
 \ELBO(\thetab,\rmq_{ei};\xb_{ei},\ab_e):=
\EE{\rmq_{ei}(\bz)}{\log \rmp_{\bA,\sigma^2}(\bx_{ei} \mid \bz)} 
- \KL{\rmq_{ei}(\bz)}{\rmp_\gammab(\bz \mid \ba_e)}. 
\end{multline*}
\looseness-1 The variational posteriors often come from a parametric family $\{\rmq_{\phib}(\bz):\phib\in\Phi\}$.
Provided $\rmp_\gammab$ is differentiable w.r.t.~$\gammab$,
black-box VI~\citep{ranganath2014black} or amortized VI with reparametrization trick~\citep{kingma2014auto,rezende2014stochastic}  yield differentiable objectives amenable to gradient-based optimization w.r.t.\ both $\thetab$ and $\phib$. 
The learned variational posterior is then used as approximation of~\eqref{target-posterior}.

\looseness-1 For a single domain and a parameter-free standard Gaussian prior, $p_\gammab(\zb)=\Ncal(\bm0, \Ib)$, the above reduces to a vanilla (linear) VAE~\citep{kingma2014auto} or probabilistic PCA~\citep{Tipping1999PPCA}.
More expressive choices for $\rmp_\gammab(\zb)$ include, e.g., a mixture of Gaussians
~\citep{dilokthanakul2016deep}, whose parameters~$\gammab$ are learned jointly with the measurement $(\bA, \sigma^2)$ and variational $\phib$ parameters, in the spirit of $\rmg$-modeling. 
Alternatively, the prior can be expressed directly in terms of the variational posterior~\citep{tomczak2018vae}, which leverages the fixed-point identity for $\rmp^\star(\bz\mid\ba)$ implied by \Cref{assump:causal-eb}.

\looseness-1 If we observe multiple domains and do not impose additional structure on $\rmp_\gammab(\zb\mid\ab)$, we obtain a conditional VAE~\citep{sohn2015learning}.
The iVAE~\citep{khemakhem2020variational} is a conditional VAE in which $\rmp_\gammab(\zb\mid\ab)$ is factorized and from an exponential family whose natural parameters depend on~$\ab$. 
It is identifiable if the set of observed domains~$\Acal$ is sufficiently diverse.
For single-cell data, perturbations $\ab_e$ are commonly modeled at the level of $\rmp_\gammab(\zb\mid\ab_e)$ as additive mean shifts~\citep{lotfollahi2023predicting,bereket2023modelling,von2025representation}.

\looseness-1 
To parametrize the interventional prior from~\cref{ass:SCM-prior}, let
\begin{equation}
\label{eq:causal_prior_g_modeling}
z_{eij} = \gamma_j\big(a_{ej},\, \bz_{ei\pa(j)},\,  u_{eij}\big), 
\,\,\, u_{eij}\iid \textsc{Unif}[0,1],
\end{equation}
where each $\gamma_j$ is a nonlinear function that models the causal mechanism from~\eqref{intervention mechanism} as $\gamma_j(0,\cdot,\cdot)\!=\!f_j(\cdot,\cdot)$ and $\gamma_j(1,\cdot,\cdot)\!=\!h_j(\cdot,\cdot)$. Several CRL methods amount to some form of $\rmg$-modeling with a prior as in~\eqref{eq:causal_prior_g_modeling}~\citep{brehmer2022weakly,zhang2024identifiability,yang2021causalvae,von2023nonparametric}.

\end{document}